\documentclass{article}

\usepackage[margin=1in]{geometry}
\usepackage{graphicx} 

\usepackage{microtype}
\usepackage{graphicx}
\usepackage{subfigure}
\usepackage{booktabs} 

\usepackage{hyperref}

\usepackage{amsmath}
\usepackage{amssymb}
\usepackage{mathtools}
\usepackage{amsthm}

\usepackage[capitalize,noabbrev]{cleveref}

\usepackage{float}

\usepackage{times}

\usepackage{fancyhdr}
\usepackage{xcolor} 
\usepackage{algorithm}
\usepackage{algorithmic}
\usepackage{natbib}
\usepackage{eso-pic} 
\usepackage{forloop}
\usepackage{url}

\theoremstyle{plain}
\newtheorem{theorem}{Theorem}[section]
\newtheorem{proposition}[theorem]{Proposition}
\newtheorem{lemma}[theorem]{Lemma}

\theoremstyle{definition}

\newtheorem{assumption}[theorem]{Assumption}
\theoremstyle{remark}

\title{Convergence Guarantees for the DeepWalk Embedding on Block Models}
\author{Christopher Harker\\
University of Utah\\
\texttt{chris.harker@utah.edu}
\and Aditya Bhaskara\\
University of Utah\\
\texttt{bhaskaraaditya@gmail.com}
}
\date{}

\usepackage{amsmath,amsfonts,amsthm,amssymb,bbm,bm,bbold}

\newcommand{\norm}[1]{\| #1 \|}
\newcommand{\iprod}[1]{\langle  #1 \rangle}
\newcommand{\su}[1]{^{(#1)}}

\def\eqref#1{equation~\ref{#1}}

\def\1{\bm{1}}

\def\eps{{\epsilon}}

\def\vone{{\bm{1}}}

\def\ve{{\bm{e}}}

\def\vp{{\bm{p}}}

\def\vu{{\bm{u}}}

\def\vw{{\bm{w}}}
\def\vx{{\bm{x}}}
\def\vy{{\bm{y}}}
\def\vz{{\bm{z}}}

\def\evx{{x}}
\def\evy{{y}}

\def\mA{{\bm{A}}}
\def\mB{{\bm{B}}}
\def\mC{{\bm{C}}}
\def\mD{{\bm{D}}}
\def\mE{{\bm{E}}}

\def\mG{{\bm{G}}}

\def\mI{{\bm{I}}}
\def\mJ{{\bm{J}}}

\def\mL{{\bm{L}}}
\def\mM{{\bm{M}}}

\def\mP{{\bm{P}}}
\def\mQ{{\bm{Q}}}
\def\mR{{\bm{R}}}

\def\mV{{\bm{V}}}
\def\mW{{\bm{W}}}
\def\mX{{\bm{X}}}
\def\mY{{\bm{Y}}}

\DeclareMathAlphabet{\mathsfit}{\encodingdefault}{\sfdefault}{m}{sl}
\SetMathAlphabet{\mathsfit}{bold}{\encodingdefault}{\sfdefault}{bx}{n}

\def\sV{{\mathbb{V}}}

\def\emA{{A}}

\def\emC{{C}}

\newcommand{\R}{\mathbb{R}}

\begin{document}

\maketitle

\begin{abstract}
Graph embeddings have emerged as a powerful tool for understanding the structure of graphs. Unlike classical spectral methods, recent methods such as DeepWalk, Node2Vec, etc. are based on solving nonlinear optimization problems on the graph, using local information obtained by performing random walks. These techniques have empirically been shown to produce ``better'' embeddings than their classical counterparts. However, due to their reliance on solving a nonconvex optimization problem, obtaining theoretical guarantees on the properties of the solution has remained a challenge, even for simple classes of graphs. In this work, we show convergence properties for the DeepWalk algorithm on graphs obtained from the Stochastic Block Model (SBM). Despite being simplistic, the SBM has proved to be a classic model for analyzing the behavior of algorithms on large graphs. Our results mirror the existing ones for spectral embeddings on SBMs, showing that even in the case of one-dimensional embeddings, the output of the DeepWalk algorithm provably recovers the cluster structure with high probability.
\end{abstract}

\section{Introduction}
\label{introduction}

Inspired by the Skip-gram model \citep{mikolov2013efficient, mikolov2013distributed} and related word embedding algorithms in the field of natural language processing, \citet{perozzi2014deepwalk}, \citet{tang2015line}, and \citet{grover2016node2vec} developed new methods to embed the nodes of a graph into geometric space. 
These methods treat random walks as akin to sentences and construct analogues to $n$-grams by sliding a fixed size window along the random walk. The probability of a node $j$ appearing in the context of a node $i$ is modeled as a function of the nodes' embeddings. This leads to a vector representation of each node in $d$ dimensions, which is then useful in addressing several machine learning problems as input features to downstream tasks, including node classification \citep{hamilton2017graphsage, perozzi2014deepwalk, grover2016node2vec, tang2015line}, link prediction \citep{grover2016node2vec, tang2015line, backstrom2011links} and community detection \citep{wang2017community, barot2021community, zhang2021consistency, davison2023community}. 
 
Predating nonlinear embedding methods, spectral embeddings have long been a classic topic in theoretical computer science and mathematics. Early algorithms for graph clustering and partitioning use the top eigenvectors of a graph's Laplacian matrix as node embeddings \citep{hall, alon, jerrum-sinclair, linial-london-rabinovich, mcsherry, spielman-teng-spectral-partitioning, arora-rao-vazirani} and the properties of spectral embeddings have been thoroughly studied in works such as \citet{belkin2001eigenmaps}, \citet{Ng2001OnSC}, \citet{von2007tutorial}, and \citet{rohe2011spectral}.

Most of the existing theoretical results surrounding embeddings produced by DeepWalk \citep{perozzi2014deepwalk} and node2vec \citep{grover2016node2vec} reframe the algorithm as a matrix factorization problem, then use methods common in the analysis of spectral embeddings to study their properties \citep{barot2021community, zhang2021consistency, qiu2018network}. \citet{levy2014advances} show that Skip-gram with negative sampling (SGNS) is implicitly performing matrix factorization of a shifted point-wise mutual information (PMI) matrix when the embedding dimension is at least as large as the number of nodes in the graph. This result inspired analyses of the properties of DeepWalk and node2vec embeddings by showing that they are also performing matrix factorization of a shifted PMI matrix \citep{qiu2018network}. The works of \citet{zhang2021consistency} and \citet{barot2021community} then show that a spectral decomposition of this matrix can be used to recover communities in graphs drawn from stochastic block models.

A major drawback of these prior works is that the matrix factorization characterization holds \emph{only when} the embedding dimension is large ($>n$, the number of nodes in the graph!), which is not usually the case in practice \citep{perozzi2014deepwalk}. Furthermore, the matrix factorization formulation of DeepWalk and node2vec is not generally used in practice, but rather the embeddings are learned through optimizing an objective using gradient descent. In this work, we answer the fundamental question: \emph{can the dynamics of gradient descent be formally analyzed for low-dimensional embeddings of natural graph classes?}

One key challenge in such analyses is the nonlinearity and nonconvexity of their objectives. When the embedding dimension is at least as large as the number of nodes in the graph, it turns out that a locally optimum solution has a  closed form structure, and its properties can be analyzed \citep{zhang2021consistency, barot2021community, levy2014advances, qiu2018network}. 

However, since most applications use only constant-dimensional embeddings, these results do no apply. Recently, the work by \citet{harker2023structure} analyzes the objective function for low-dimensional embeddings, and shows that for graphs drawn from the stochastic block model (SBM), the DeepWalk objective has a global minimum that has a well-clustered structure. However, their work is restricted to the case of 2-block SBMs, and more importantly, does not study the question of whether gradient descent (or any other heuristic) converges to the optimal solution. Another recent work of \citet{davison2023community} examines the asymptotic behavior of node2vec embeddings learned by minimizing the SGNS objective on graphs obtained fom SBMs. But once again, they do not answer the algorithmic question of how to obtain the optimal solution (e.g., via gradient descent).  

In contrast to these works, we analyze the dynamics of the gradient descent update procedure and show that when the graph is drawn from a symmetric SBM, the embedding vectors of nodes within a block are clustered together, while the vectors corresponding to nodes in different blocks are farther apart. We perform the analysis for the DeepWalk procedure, which minimizes a nonconvex objective obtained by performing random walks. 

Our work is inspired by recent works that study the theoretical properties of the t-SNE algorithm, which shares many of the same challenges as DeepWalk due to its nonlinear and nonconvex objective. The works of \citet{linderman2019clustering}, \citet{arora2018analysis}, and \citet{cai2022theoretical} examine the dynamical properties of t-SNE's gradient descent updates and show that similar data points are clustered together while separating from dissimilar data points. 

\subsection{Our Results}\label{sec:our-results}
We consider graphs drawn from a stochastic block model (SBM). In the simplest setting, we have three parameters, an integer $K \ge 2$ (number of blocks) and probabilities $p, q \in (0,1)$. The vertices are divided into $K$ parts or \emph{clusters} at random, and an edge is placed between two vertices in the same cluster with probability $p$, and between vertices in two different clusters with probability $q$. We assume $p>q$, and that both parameters are $> n^{\rho-1}$ for some parameter $\rho \in (0,1)$. For details about the SBM and graph generation, we refer to Section~\ref{sec:background-sbm}. Our main result can be stated as follows.
\begin{theorem}[Informal]
Given a graph $G$ drawn from an SBM with $K$ blocks and parameters $p$ and $q$, DeepWalk embeddings, obtained by initializing a solution in small enough ball $\lVert \vw^{(0)} \rVert \le \epsilon$, and training with gradient descent with learning rate $\eta > 0$ small enough, approximately recovers communities with high probability.
\end{theorem}

For more formal statements, see theorems~\ref{thm:cluster-inside} and~\ref{thm:main-separation}, which respectively upper bound the spread of embeddings within clusters and lower bound the separation of embeddings across clusters. Theorem \ref{thm:weak_recovery} then gives a bound on the fraction of each community that can be recovered (thus formalizing the notion of approximate recovery above). The main technical challenge in our result is reasoning about the nonlinear gradient update rule. The idea in our proofs is to show that if we initialize our solution in a small enough ball, the gradient update is close enough to a ``linearized'' update, which results in a sufficiently good cluster structure for the overall solution. This is similar in spirit to other recent work that show that small random initializations result in dynamics that are like spectral updates (see~\cite{stoger2021small,satpathi2021dynamics} and references therein).

We also remark that we prove our main results even for the case of one-dimensional embeddings. We find this somewhat surprising because typically (e.g., for spectral algorithms), separation results for SBMs with $K$ clusters hold only when the embedding dimension is $\ge K$. We expect our analysis to extend to the case of higher dimensional embeddings (as the 1D analysis can be applied to each coordinate).

One weakness of our result is that unlike traditional results for recovery guarantees for SBM, we do not give a precise characterization in terms of the difference $(p-q)$ and recovery accuracy. Instead, our arguments rely on a closeness assumption between the empirical and \emph{expected} co-occurrence matrices, stated as Assumption~\ref{ass:co-occurrence-error}. It is an interesting open direction to study tight recovery guarantees in terms of $p,q,K$.

On the experimental side, we validate our results: we show a clear separation between the embeddings of vertices across clusters for different choices of the embedding dimension. We also show that the approximation via linearization potentially holds for a fairly large regime of parameters, even beyond the ones we use for our theoretical guarantees.

\section{Background}
\label{background}

\subsection{Basic Notation}\label{sec:notation}

We start with common notation that we use throughout the paper.  For a column vector $\vx = (\evx_1, ..., \evx_n) \in \mathbb{R}^n$, we define the $\ell_2$ norm as $\lVert \vx \rVert = \sqrt{\sum_{i=1}^n x_i^2}$ and the $\ell_{\infty}$ norm as $\lVert \vx \rVert_{\infty} = \max_{1 \le i \le n} | \evx_i |$. We denote by $diag(\vx) \in \mathbb{R}^{n \times n}$ the diagonal matrix whose $i^{th}$ diagonal entry is $\evx_i$. We denote the all-ones vector as $\mathbf{1} = (1, ..., 1) \in \mathbb{R}^{n}$. The all-ones matrix (of dimensions $n\times n$ unless specified otherwise) will be denoted by $\mJ$. We denote the dot product between two vectors $\vx, \vy \in \mathbb{R}^n$ as $\langle \vx, \vy \rangle$.

For a matrix $\mA \in \mathbb{R}^{n \times n}$ we define the spectral norm as $\lVert \mA \rVert = \max_{\lVert \vx \rVert = 1} \lVert \mA \vx \rVert$, its $\ell_{\infty}$ norm as $\lVert \mA \rVert_{\infty} = \max_{1 \le i \le n} \sum_{j=1}^n \left| \emA_{ij} \right|$, its max norm as $\lVert \mA \rVert_{max} = \max_{1 \le i,j \le n} | \emA_{ij}|$, and its Frobenius norm as $\lVert \mA \rVert_F = \newline
\sqrt{\sum_{i=1}^n \sum_{j=1}^n \emA_{ij}^2}$. We define the degree operator $\mD : \mathbb{R}^{n \times n} \rightarrow \mathbb{R}^{n \times n}$ as $\mD(\mA) = diag(\mA \mathbf{1})$. At times we use $\mD_\mA$ in place of $\mD(\mA)$ when it is more convenient.  

For a graph $G = (V, E)$ with $n$ nodes and $|E|$ edges, we denote the adjacency matrix of the graph as $\mA$. For a node $i$, its degree is $d_i = \sum_{j=1} \emA_{ij}$. We let $\mP = \mD_{\mA}^{-1}\mA$ denote the transition matrix of the graph. 

\subsection{The DeepWalk Algorithm}
Let $G = (V, E)$ be a graph with $n$ nodes. The DeepWalk algorithm consists of two main parts \citep{perozzi2014deepwalk}. First, a \emph{co-occurrence matrix} is constructed using random walks on the graph. Second, two embedding vectors are learned for every vertex in the graph. These are referred to as the \emph{node} and \emph{context} embeddings; they are learned by minimizing a nonconvex objective function.\footnote{While having separate word and context embeddings have intuitive meaning in the language context, the difference is not so clear for graphs. In some implementations of DeepWalk, the same embedding is used. We use separate embeddings to remain faithful to the original formulation.}
\paragraph{Constructing the Co-occurrence Matrix.}\label{sec:co-occurrence} 
Given a graph $G = (V, E)$, the algorithm first performs $r$ random walks of length $L$. For each random walk, a window of size $T$ slides along the generated path. Let $w^{(m)}$ denote the path of $L$ nodes generated by the $m^{th}$ random walk, and let $w^{(m)}_k$ denote the $k^{th}$ step of the $m^{th}$ random walk. Furthermore, we assume that the starting node of each walk is sampled from the stationary distribution $\Pr[w_1^{(m)} = i] = \frac{d_i}{2|E|}$ over the graph $G$. 

As the window of size $T$ is slid along the path, the entries of a matrix $\mC$ are updated. The entries $\emC_{ij}$ of this matrix contain the number of times that a node $j$ appears in the context window of node $i$. Formally,

\begin{align*} \emC_{ij} &= \sum_{t=1}^T \sum_{m=1}^r \sum_{k=1}^{L-t} \mathbb{1}\{w_k^{(m)} = i, w_{k+t}^{(m)} = j\} + \sum_{t=1}^T \sum_{m=1}^r \sum_{k=1}^{L-t} \mathbb{1}\{w_k^{(m)} = j, w_{k+t}^{(m)} = i\}.
\end{align*}
Limiting forms of the co-occurrence matrix are explored in many prior works. Different variations of the limiting form exist depending on whether the length of the walk $L$ goes to $\infty$ \citep{qiu2018network}, whether the number of walks $r$ goes to $\infty$ \citep{zhang2021consistency, barot2021community} or whether the window size goes to $\infty$ \citep{musco2020infinitywalk}. Obtaining more quantitative concentration bounds on this matrix have also been studied \citep{qiu2020matrix, kloepfer2021delving}. Proofs of the following lemma can be found in \citet{zhang2021consistency}, \citet{barot2021community}, and \citet{harker2023structure}.  
\begin{lemma}\label{lem:co-occur-ij}
    Let $\mA$ be an adjacency matrix of a fixed graph $G$ and let $w_k^{(m)}$ denote the $k^{th}$ step of the $m^{th}$ random walk generated by the DeepWalk algorithm. Let $\pi_i = \frac{d_i}{2|E|}$ and let $(\mP^t)_{ij} = \Pr[w_{t+1} = j | w_1 = i]$. Then as $r \rightarrow \infty$, 
    \begin{equation}\label{eq:coc-conc}
        \frac{\emC_{ij}}{r} \xrightarrow{a.s} 2\sum_{t=1}^T (L-t) \cdot \pi_i(\mP^t)_{ij}.
    \end{equation}
\end{lemma}
\paragraph{Computing Embeddings}
Given a co-occurrence matrix $\mC$, we compute $d$-dimensional node embeddings $\mX, \mY \in \mathbb{R}^{n \times d}$, where each row $\vx_i, \vy_i$ are the node and context representations of vertex $i$, by minimizing the following objective function:
\begin{align}
    \mathcal{L}(\mC; \mX, \mY) &= -\sum_{i=1}^n \sum_{j=1}^n \emC_{ij} \log{\left(\frac{\exp{\left( \langle \vx_i, \vy_j \rangle \right)}}{\sum_{k=1}^n \exp{\left( \langle \vx_i, \vy_k \rangle \right)}} \right)}. \label{objective}
\end{align}
This is typically done using gradient descent. At any given iteration $t$ of gradient descent, let $\mQ^{(t)}$ be the matrix defined by \begin{equation}\mQ^{(t)}_{ij} = \frac{\exp{\left( \langle \vx^{(t)}_i, \vy^{(t)}_j \rangle \right)}}{\sum_{k=1}^n \exp{\left( \langle \vx^{(t)}_i, \vy^{(t)}_k \rangle \right)}}.\end{equation} For a co-occurrence matrix $\mC$ and a learning rate $\eta > 0$, the update equations in matrix form are
\begin{align}
    \mX^{(t+1)} &= \mX^{(t)} - \eta\left(\mD_{\mC} \mQ^{(t)} - \mC \right) \mY ,\\
    \mY^{(t+1)} &= \mY^{(t)} - \eta\left(\mD_{\mC} \mQ^{(t)} - \mC \right)^{\top} \mX.
\end{align}
If we let $\mW^{(t)} = \begin{bmatrix} \mX^{(t)} \\ \mY^{(t)} \end{bmatrix}$ and $\mG^{(t)} = \mD_{\mC}\mQ^{(t)} - \mC$. Then we can write the updates jointly as
\begin{equation}
    \mW^{(t+1)} = \begin{bmatrix} \mI & -\eta \mG^{(t)} \\ -\eta \mG^{(t)^{\top}} & \mI \end{bmatrix} \mW^{(t)}.  \label{w_update}
\end{equation}
The main challenge in the analysis is working with the matrices $\mQ\su{t}$, whose entries change  as we update the embeddings. To keep the analysis clean, we will work with the case $d=1$, as discussed earlier. Here, the embeddings are defined simply by vectors $\vx$ and $\vy$, but the essential difficulty (of dealing with $\mQ$) remains.

\subsection{Stochastic Block Models}\label{sec:background-sbm}

The stochastic block model (SBM) \citep{Holland1983StochasticBM} is a generalization of Erd\H{o}s-Renyi random graphs. This model naturally generates graphs containing communities; therefore, it has been a popular choice of generative model studied in the theoretical analysis of community recovery algorithms (see, e.g. \citet{abbe2017community} and references therein). 

A $K$ block stochastic block model (SBM) generates a random graph $G=(V,E)$ through a simple procedure. First, each node is first assigned to one of $K$ blocks. We refer to $\sV_k$ as the set of vertices that belong to community $k$. We define a community membership matrix $\mathbf{\Theta} \in \{0, 1\}^{n \times K}$ where its entries $\Theta_{ik} = 1$ if node $i$ belongs to community $k$ and is $0$ otherwise. Next, edges are assigned to each pair of nodes. We define a symmetric matrix $\mB \in [0,1]^{K \times K}$ whose entries $B_{ij}$ denote the probability of a node in cluster $i$ being connected to a node in cluster $j$. Then the matrix $\widetilde{\mB} = \mathbf{\Theta} \mB \mathbf{\Theta}^{\top}$ is a block matrix of probabilities defined by the community membership matrix $\mathbf{\Theta}$. In the remainder of the paper, we assume that the vertex indices are permuted such that the first $n/K$ correspond to the first cluster, the next $n/K$ to the second cluster, and so on. In this case, we can also write $\widetilde{\mB} = \mB \otimes \mJ_{\frac{n}{K} \times \frac{n}{K}}$, where $\otimes$ denotes the Kronecker product. Given a probability matrix $\widetilde{\mB}$, edges of the graph $G(V, E)$ are then independent Bernoulli random variables with $A_{ij} \sim Bern(\widetilde{B}_{ij})$, and $A_{ij} = A_{ji}$ for all $i < j$. Therefore, the probability of node $i$ and node $j$ being connected depends only on the communities to which $i$ and $j$ belong. To be consistent with later notation, we also define the expected adjacency matrix $\overline{\mA}$ as the matrix whose entries are the expected values of the corresponding entries in the adjacency matrix $\mA$; by definition, $\overline{\mA} = \widetilde{\mB}$.

In our analysis, we assume that the number of communities $K$ is fixed and the communities are of equal size: $|\sV_k| = \frac{n}{K}$. We also assume that the matrix $\mB$ has diagonal entries equal to $p$ and off-diagonal entries equal to $q$. This makes the probability matrix $\widetilde{\mB}$ a block matrix with $K$ equally sized blocks. It has blocks of $p$ along the diagonal and blocks of $q$ off the diagonal. 

\section{Analysis}
\label{analysis}
We break up the analysis into three main parts. First, we will show properties of the co-occurrence matrix that will be important for the analysis, describe the algorithm, and set up the main notation for the analysis. Second, we show the main step of ``linear approximation'', where we argue that the gradient descent update can be expressed as a linear update plus an error term that is controlled by the length of the embedding solution. Finally, we analyze the gradient descent dynamics, and show the desired properties of the final solution. 

For simplicity, we assume that $n$ is large, and $K$ is a constant. We also assume (for the analysis) that when writing down the co-occurrence matrix $\mC$, the vertices are permuted so that the blocks $\sV_1, \sV_2, \dots, \sV_K$ appear together (thus leading to the form of $\overline{\mC}$ below).

\subsection{Co-occurrence Structure and Algorithm}\label{sec:structure_algo}

Assume that we have a graph drawn from the symmetric SBM with $K$ blocks and parameters $p$ and $q$, and let $\mC$ be the symmetric co-occurrence matrix obtained using random walks as described in Section~\ref{sec:co-occurrence}. Our first step will be to prove that $\mC$ is spectrally close to the matrix $\overline{\mC}$ defined as
\begin{equation}\label{eq:expected_C_lemma}
    \overline{\mC} = 2 \sum_{t=1}^T \frac{(L-t)}{n\overline{d}} \mD_{\overline{\mA}}\overline{\mP}^t,
\end{equation}
where $\overline{\mP} = \mD_{\overline{\mA}}^{-1}\overline{\mA}$ and $\overline{d} = \frac{n}{K}p + \frac{n(K-1)}{K}q$ is the expected degree of the graph. Since the expected adjacency matrix $\overline{\mA}$ has a block structure (as described in Section \ref{sec:background-sbm}), the matrix $\overline{\mC}$ also has a block structure. In other words, for some parameters $a, b$,
\begin{equation}\label{eq:expected_C}
\overline{\mC} = \underbrace{\begin{bmatrix} a & b & ... & b \\ b & a & ... & b \\ \vdots & \vdots & \ddots & \vdots \\ b & b & ... & a \end{bmatrix}}_{K \times K} \otimes ~\mJ_{\frac{n}{K} \times \frac{n}{K}}.
\end{equation}
As before, $\otimes$ denotes the  Kronecker product, making $\overline{\mC}$ a block matrix with $K \times K$ blocks, each of size $\frac{n}{K} \times \frac{n}{K}$, with $a$ in the diagonal blocks and $b$ in the off-diagonal blocks. In the case of symmetric SBMs, we can express $a$ and $b$ explicitly in terms of the SBM parameters $p$ and $q$ (see Appendix \ref{sec:cooccurrence_prop}). 

The following lemma shows that $\mC$ is spectrally close to the matrix $\overline{\mC}$. The proof can be found in Appendix \ref{sec:cooccurence_conc}.

\begin{lemma}\label{lem:cooccurrence_conc}
Suppose $\mC$ is an $n \times n$ co-occurrence matrix constructed as in Lemma \ref{lem:co-occur-ij} from a graph $G$ drawn from an SBM with $K$ blocks and parameters $p > q \ge n^{\rho - 1}$, for some parameter $\rho \in (0,1)$, and that the matrix $\overline{\mC}$ is defined as in Equation (\ref{eq:expected_C_lemma}). Then for appropriate parameters $a, b$ (defined in terms of $p,q$), we have $\norm{\mC - \overline{\mC}} \le c  \lVert \overline{\mC} \rVert \sqrt{\frac{\log{n}}{n^{\rho}}}$ for some absolute constant $c$.
\end{lemma}

This lemma allows us to reason about the eigenvalues of $\mC$ using those of $\overline{\mC}$, via Weyl's Theorem and other matrix perturbation bounds. Suppose $\lambda_1 \ge \lambda_2 \ge \dots \ge \lambda_n$ are the eigenvalues of $\overline{\mC}$. Then, 
\[ \lambda_i = \begin{cases} \frac{n}{K}(a + (K-1)b) &\text{ if } i=1, \\ \frac{n}{K}(a-b) &\text{ if } i=2,...,K, \\ 0 &\text{ if } i=K+1, ..., n. \end{cases} \]
The assumption that $p>q\ge n^{\rho-1}$ made in Lemma \ref{lem:cooccurrence_conc} will ensure that the nonzero eigenvalues $\lambda_1, ..., \lambda_K$ are all $\Theta(\frac{1}{n})$, assuming that the length of the random walk $L$ and number of communities $K$ are constant (see Appendix \ref{sec:cooccurrence_prop}). As a result, the spectral norm $\lVert \overline{\mC} \rVert = O\left(\frac{1}{n}\right)$, which implies that $\lVert \mC - \overline{\mC} \rVert = \widetilde{O}\left(\frac{1}{\sqrt{n^{\rho+2}}}\right)$. Therefore, in order to simplify our calculations, we make the following assumption throughout the remainder of the paper:

\begin{assumption}\label{ass:co-occurrence-error}
Suppose $\mC$ is the $n \times n$ symmetric co-occurrence matrix obtained via random walks of length $L=O(1)$ as in Lemma \ref{lem:co-occur-ij}, and $\overline{\mC}$ is defined as in Equation (\ref{eq:expected_C_lemma}). Then we assume that $\mC = \overline{\mC} + \mR$ for an error matrix $\mR$ that is symmetric and satisfies $\norm{\mR} < \frac{c}{\sqrt{n^{\rho+2}}}$ for some absolute constant $c$. 
\end{assumption}

The corresponding eigenvectors are also easy to characterize: the top eigenvector is the all-ones vector, which we write as $\frac{\vone}{\sqrt{n}}$. The next $(K-1)$ eigenvalues are all equal, so the corresponding eigenvectors are not unique. One (nonorthogonal) basis for their span is $\{ \ve_{\sV_i} - \ve_{\sV_1}\}_{i=2}^K$, where $\ve_{S}$ is a shorthand for $\vone_{S}$ (the vector that is $1$ in the $j$th position if $j \in S$ and $0$ otherwise). 

Next, the following matrix $\mM$ and the corresponding $\overline{\mM}$ play a key role in our analysis:
\[ \mM := \mD_{\mC} \frac{\mJ}{n} - \mC, \quad \overline{\mM} := \mD_{\overline{\mC}} \frac{\mJ}{n} - \overline{\mC}, \]
where $\mD$ denotes the degree operator (see Section~\ref{sec:notation}). Since $\mC - \overline{\mC} = \mR$ and the degree operator is linear, we have
\[ \mM - \overline{\mM} = \mD_{\mR} \frac{\mJ}{n} - \mR = - \left( \mI - \frac{\mJ}{n} \right) \mR.  \]
This implies that $\norm{\mM - \overline{\mM}} \le \norm{\mR} < \frac{c}{\sqrt{n^{\rho+2}}}$. Now the eigenvalues of $\overline{\mM}$ are easy to see: along the all-ones vector, the eigenvalue becomes $0$, and all the other eigenvalues stay the same. Thus, $\overline{\mM}$ has exactly $(K-1)$ nonzero eigenvalues, all equal to $\frac{n}{K} (b-a)$.

For concreteness, we provide the gradient descent procedure that we analyze in Algorithm \ref{alg:gradient_descent}. We initialize the embeddings randomly and normalize them so that their norm is small. Then we run gradient descent until the norm of the embeddings reaches a predefined size.\footnote{The parameters $\epsilon$ and $\Delta$ can be expressed more generally in terms of $\rho$, i.e., $\epsilon = n^{-c\rho}$ and $\Delta = n^{c' \rho}$ where $c, c'$ are positive constants.}

\begin{algorithm}[tb]
   \caption{DeepWalk Gradient Descent}
   \label{alg:gradient_descent}
\begin{algorithmic}
   \STATE {\bfseries Input:} Co-occurrence matrix $\mC \in \mathbb{R}^{n \times n}$; learning rate $\eta > 0$; parameters $\eps = \frac{1}{n^{2/3}}$ and $\Delta = n^{1/6}$.
   \STATE {\bfseries Initialize:} $t=0$; $\vx^{(0)}, \vy^{(0)} \in \mathbb{R}^{n}$ with $\evx_i, \evy_i \sim \mathcal{N}(0,1)$ for all $i \in [n]$; $\vw^{(0)} = \begin{bmatrix} \vx^{(0)} \\ \vy^{(0)} \end{bmatrix}$
   \STATE Normalize $\vw^{(0)}$: $\vw^{(0)} \leftarrow \frac{\vw^{(0)}}{\lVert \vw^{(0)} \rVert} \eps$
   \REPEAT
   \STATE $t \leftarrow t + 1$;
   \STATE Compute $\mQ^{(t)}$ where $Q^{(t)}_{ij} = \frac{\exp{\left( \langle \evx_i^{(t)}, \evy_j^{(t)} \rangle \right)}}{\sum_{k=1}^n \exp{\left( \langle \evx_i^{(t)}, \evy_k^{(t)}\rangle \right) }}$ for all $i,j \in [n]$;
   \STATE $\mG^{(t)} = \mD_{\mC}\mQ^{(t)} - \mC$;
   \STATE $\vw^{(t)} \leftarrow \begin{bmatrix} \mI & -\eta \mG^{(t)} \\ -\eta \mG^{{(t)}^{\top}} & \mI \end{bmatrix} \vw^{(t)}$
   \UNTIL{$\lVert \vw^{(t)} \rVert \ge \eps \Delta$}
   \STATE $t_f = t$;
   \STATE {\bfseries Return:} $\vw^{(t_f)}$; $t_f$
\end{algorithmic}
\end{algorithm}

\subsection{Linear Approximation of Gradient Update}\label{sec:linear_approx_analysis}
Next, we show that the matrix $\mQ\su{t}$ in Equation (\ref{w_update}) (gradient descent step) can be approximated simply by the all-ones matrix (suitably scaled). This allows us to obtain a linear approximation to the $\vw$ update.

\begin{proposition}\label{prop:1}
Let $\vx, \vy \in \R^{n}$ and $\vw = \begin{bmatrix}
    \vx \\ \vy
\end{bmatrix}$, and suppose that $\norm{\vw}< \eps$ for some $\eps \in (0, 1/2)$. Let $\mQ$ be the matrix defined by $\mQ_{ij} = \frac{\exp{\left( \langle \vx_i, \vy_j \rangle \right)}}{\sum_{k=1}^n \exp{\left( \langle \vx_i, \vy_k \rangle \right)}}$. Then
\[ \left\| \mQ - \frac{1}{n} \mJ \right\|_F \le \epsilon^2. \]
\end{proposition}
The proof is deferred to section~\ref{proof:prop1}. It is important to note that $\eps^2$ is a bound on the Frobenius norm (not the square of the Frobenius norm). This implies that if the node and context embeddings exist in a small enough ball, then the softmax matrix $\mQ^{(t)}$ behaves like the fixed matrix $\frac{1}{n} \mJ$.

Utilizing this observation, we can rewrite the update Equation (\ref{w_update}) in terms of a linear term and an ``error'' term. Recall that $\mM = \mD_{\mC} \frac{\mJ}{n} - \mC$. Thus we can write~(\ref{w_update}) as
\begin{equation}\label{eq:main-update}
\vw\su{t+1} = \begin{bmatrix}\mI & -\eta \mM \\ -\eta \mM^\top & \mI  \end{bmatrix} \vw\su{t} - \eta \mE\su{t} \vw\su{t},
\end{equation}
where 
\[  \mE\su{t} = \begin{bmatrix} \mathbf{0} & \mD_{\mC} (\mQ\su{t} - \frac{1}{n}\mJ) \\ \left(\mD_{\mC} (\mQ\su{t} - \frac{1}{n}\mJ)\right)^{\top} & \mathbf{0} \end{bmatrix}. \] 
The following is a consequence of Proposition~\ref{prop:1} and the choice of $a$ and $b$ (see Appendix~\ref{proof:error-norm}).
\begin{lemma}\label{lem:error-norm}
    Suppose the iterate $\vw\su{t}$ satisfies $\norm{\vw\su{t}} \le \eps$ for some $\eps \in (0, 1/2)$. Then $\norm{\mE\su{t}}_F \le 4 \eps^2$.
\end{lemma}

Next, let denote the ``linear portion'' of the update by $\mL$. In other words, we define
\begin{equation}\label{eq:lin-update-definition}
\mL := \begin{bmatrix} \mI & -\eta\mM \\ -\eta\mM^\top & \mI \end{bmatrix}. 
\end{equation}

We have the following observations about the spectrum of $\mL$. The proof is deferred to Appendix \ref{sec:cL-properties}.
\begin{lemma}\label{lem:cL-properties}
Suppose $\mL$ is defined as in (\ref{eq:lin-update-definition}). Then $\mL$ has precisely $(K-1)$ eigenvalues that are $> (1+\eta \gamma)$, where $\gamma = \frac{n (a-b)}{2K}$. All the other eigenvalues are $< (1+\eta \frac{c}{\sqrt{n^{\rho+2}}})$. Finally, all the eigenvalues are in the interval $(1- 4\eta \gamma, 1+4 \eta \gamma)$, for $\gamma$ as above.
\end{lemma}

In what follows, let $\Pi$ be the projector onto the span of the eigenvectors corresponding to the top $(K-1)$ eigenvalues of $\mL$. 

We begin by obtaining bounds on the number of iterations performed by the algorithm.
\begin{lemma}\label{lem:num-iterations}
Let $t_f$ be the number of iterations performed by Algorithm \ref{alg:gradient_descent}. With probability at least $0.9$ over the choice of the initialization, we have
\[ \frac{1}{\eta} < t_f < \frac{4\log (n/\Delta)}{\eta}. \]
\end{lemma}
\begin{proof}
Using Lemma~\ref{lem:cL-properties} along with the fact that $4\gamma \le 1$, we have
\[ \left\| \mL \right\| \le 1+\eta.  \]
Now, as long as $\norm{\vw\su{t}} \le \eps \Delta$, we have $\norm{\mE\su{t}} \le 4 (\eps \Delta)^2$ using Lemma~\ref{lem:error-norm}. By the choice of parameters $\eps, \Delta$, we will ensure that $4 (\eps \Delta)^2 < 1$. Thus, we obtain that
\[ \norm{\vw\su{t+1}} < (1+2\eta) \norm{\vw\su{t}}.  \]
Now, since $\norm{\vw\su{0}}=\eps$, $\norm{\vw\su{t_f}} \ge \eps \Delta$, and $\Delta > 8$, the number of iterations $t_f > \frac{2}{\log (1+2\eta)} > \frac{1}{\eta}$.

To see the upper bound, we use more properties of $\mL$ from Lemma~\ref{lem:cL-properties}. Specifically, suppose we define $\vz\su{t} = \Pi \vw\su{t}$ where $\Pi$ is the projector onto the span of the eigenvectors corresponding to the top $(K-1)$ eigenvalues of $\mL$. We argue that this component of $\vw\su{t}$ itself grows sufficiently. First, we note that because of the randomness in the initialization, we have that $\norm{\vz\su{0}} \ge \frac{1}{10 \sqrt{n}} \norm{\vw\su{0}} = \frac{\eps}{10\sqrt{n}}$, with probability $\ge 0.9$.

Then, by multiplying the update Equation (\ref{eq:main-update}) by $\Pi$, we get
\begin{align*} \norm{\vz\su{t+1}} &\ge  (1+\eta \gamma) \norm{\Pi \vw\su{t}} - \norm{ \Pi \mE\su{t} \vw\su{t} } \\
&\ge (1+\eta \gamma) \norm{\vz\su{t}} - 4 \eta (\eps \Delta)^3 \\ 
&\ge (1+\frac{\eta\gamma}{2} ) \norm{\vz\su{t}}.
\end{align*}
For the last inequality, we used the fact that $(\eps \Delta)^3 < \frac{\gamma \eps}{20\sqrt{n}}$, which follows from our assumptions on the parameters. We are also using the fact that $\vz\su{t}$ always has norm $> \frac{\eps}{10\sqrt{n}}$ which follows inductively from the above.

This establishes that $\norm{\vz\su{t}} \ge (1+\frac{\eta\gamma}{2} )^{t_f} \norm{\vz\su{0}}$. Since $\norm{\vz\su{t}}\le \norm{\vw\su{t}} \le 2\eps \Delta$, the desired upper bound on $t_f$ follows.
\end{proof}

As the final result in this part of the proof, we argue that most of the mass of $\vw\su{t_f}$ at the end of the algorithm is on the top $(K-1)$ eigenspace of $\mL$. This is precisely what we would expect if the entire update was linear, i.e., if $\vw\su{t_f}$ where equal to $\mL^{t_f} \vw\su{0}$. It seems natural to try showing that the $\vw\su{t}$ stays close to the linearized update, $\mL^t \vw\su{0}$.  However, the error in this step turns out to be difficult to control unless we choose the parameter $\eps$ to be really small (e.g., $\frac{1}{n^2}$). In our analysis, we take a different route that lets us analyze a wider parameter range. 

The key will be to look at the difference between $\vw\su{t}$ and its projection to the top-$(K-1)$ eigenspace of the matrix $\mL$. Recall that $\Pi$ is the projection matrix onto this space. As above, define
\begin{equation}\label{def:projection} \vz\su{t} := \Pi \vw\su{t} \text{  for all $t$.}\end{equation}

\begin{lemma}\label{lem:error-w-z}
At every iteration $t$ of Algorithm \ref{alg:gradient_descent}, we have:
\[ \norm{ \vw^{(t+1)} - \vz^{(t+1)} } \le (1 + \eta \frac{c}{\sqrt{n^{\rho+2}}})\norm{ \vw^{(t)} - \vz^{(t)} } + \eta (\eps \Delta)^3.\]
Consequently, when the algorithm terminates, we have
\[ \norm{\vw\su{t_f} - \vz\su{t_f}} \le \frac{4 \norm{\vw\su{t_f}}}{\Delta}.   \]
\end{lemma}

\begin{proof}
We start with the update, Equation (\ref{eq:main-update}), and left-multiply both sides with $(\mI - \Pi)$. Thus, we get
\begin{align*}
\vw\su{t+1} - \vz\su{t+1} &= (\mI- \Pi) \vw\su{t+1} \\
&= (\mI - \Pi) \mL \vw\su{t} + \eta (\mI - \Pi) \mE\su{t} \vw\su{t}.
\end{align*}

The second term on the RHS can be bounded using
\[ \norm{ (I - \Pi) \mE\su{t} \vw\su{t}} \le \norm{\mE\su{t} \vw\su{t}}.\]
Then, using Lemma~\ref{lem:error-norm} and the fact that $\norm{\vw\su{t}} \le \eps \Delta$, we can bound this term by $\eta (\eps \Delta)^3$. For the first term, since $\mL$ commutes with $(I - \Pi)$, we have
\[ \norm{\mL (\mI - \Pi) \vw\su{t}} \le (1 + \eta \frac{c}{\sqrt{n^{\rho+2}}})\norm{ \vw\su{t} - \vz\su{t}}. \]
This follows from the property of $\mL$ from Lemma~\ref{lem:cL-properties}, that in the space orthogonal to the top $(K-1)$ eigenvectors, the spectral norm of $\mL$ is $(1+\eta \frac{c}{\sqrt{n^{\rho+2}}})$. Combining the two observations, the desired inequality follows.

To see the second part of the lemma, let us introduce some notation: write $\xi_t = \norm{\vw\su{t} - \vz\su{t}}$, $\theta = \eta \frac{c}{\sqrt{n^{\rho+2}}}$, and $\delta = \eta (\eps \Delta)^3$. So the bound above can be written as 
\[ \xi_{t+1} \le (1+\theta) \xi_t + \delta. \]
Expanding, we have
\begin{align*}
\xi_{t_f} &\le (1+\theta)^{t_f} \xi_0 + \delta \left[ 1 + (1+\theta) + \dots + (1+\theta)^{t_f-1} \right] \\
& \le (1+\theta)^{t_f} \left( \xi_0 + \delta t_f \right). 
\end{align*}
From the bound on $t_f$ from Lemma~\ref{lem:num-iterations} and since $n$ is large (so $\sqrt{n} \gg \log n$), we have that $(1+\theta)^{t_f} \le 2$, and from our setting of parameters $\eps, \Delta$, we have $\delta t_f < \eps$ and $\xi_0 \le \eps$. This implies that $\xi_{t_f} \le 4\eps$.

Since $\norm{\vw\su{t_f}} \ge \eps\Delta$ at the end of the algorithm, the desired bound follows.
\end{proof}

\subsection{Convergence Analysis}\label{sec:convergence_analysis}

As the final step, we show that the solution obtained at the end of the algorithm satisfies the desired cluster structure. We use Lemma~\ref{lem:error-w-z} to primarily argue about the vector $\vz\su{t_f}$. 

The first result says that the obtained $x$-embedding is well clustered. 

\begin{theorem}[Clustering property]\label{thm:cluster-inside}
Suppose $G$ is drawn from a symmetric SBM with $K$ communities and suppose the co-occurrence matrix $\mC$ used for constructing the embeddings satisfies Assumption~\ref{ass:co-occurrence-error} and $\rho \in \left(\frac{1}{3}, 1 \right)$. 
Suppose $\vx$ is the $x$-embedding obtained by Algorithm~\ref{alg:gradient_descent} (i.e., the first $n$ coordinates of $\vw\su{t}$). Let $\pmb{\mu}$ be the ``vector of means'', i.e., for any index $j$ in (ground truth) cluster $\sV_i$, the entry $\mu_j$ is equal to $\mu_{\sV_i} := \frac{\sum_{r \in \sV_i} x_r}{|\sV_i|}$. Then we have:
\[ \norm{ \vx - \pmb{\mu}} \le \frac{5 \norm{\vx}}{\Delta}.  \]
\end{theorem}

\paragraph{Remark.} Since $\Delta$ is large (grows as $n^{1/6}$), this implies that the variation in the embedding values \emph{within clusters} is small. Quantitatively, we expect all the entries of $\vx$ to be around $\frac{\norm{\vx}}{\sqrt{n}}$ in magnitude. Suppose they are in the interval $[-2 \frac{\norm{\vx}}{\sqrt{n}}, 2 \frac{\norm{\vx}}{\sqrt{n}}]$. The theorem says that the distance of a typical $\vx$ to the cluster center is only about $\frac{\norm{\vx}}{\Delta \sqrt{n}}$. The catch with this theorem, however, is that it does not imply that there is a separation between the embedding values \emph{across} clusters. This will be the subject of Theorem~\ref{thm:main-separation}. But first, we prove Theorem~\ref{thm:cluster-inside}
\begin{proof}
The outline of the proof is as follows: first we argue that $\vx - \pmb{\mu}$ is, in fact, precisely the projection of $\vx$ to the space orthogonal to the span of the top $K$ eigenvectors of the matrix $\overline{\mC}$ (the ``ideal'' co-occurrence matrix). This implies that
\[ \norm{\vx - \pmb{\mu}}^2 \le \norm{ (I - \overline{\Pi}) \vw\su{t_f}}^2,\] where $\overline{\Pi}$ is a projection onto the top $(K-1)$ subspace of  $\overline{\mL}$. Noting that the difference between $\Pi$ and $\overline{\Pi}$ is small (by eigenspace perturbation theorems), we can then apply Lemma~\ref{lem:error-w-z} to conclude the argument.  The details of the proof are deferred to Appendix~\ref{proof:cluster-inside}
\end{proof}

Next, we wish to prove a cluster ``gap'' property: in other words, for the obtained embedding $\vx$, $|\mu_{\sV_i} - \mu_{\sV_j}|$ is large enough. To show this, we proceed by observing that during our iterations, the projection $\vz\su{t} = \Pi \vw\su{t}$ gets updated in a very simple manner. This property is only true in the case of the \emph{symmetric} SBM.

\newcommand{\err}{\textsf{err}}
\begin{lemma}\label{lem:z-update-simple}
Let $\vz\su{t}$ be defined as in (\ref{def:projection}), and suppose we start with $\norm{\vz\su{0}} \ge \frac{\eps}{10\sqrt{n}}$. Then for all $t \le t_f = O\left( \frac{\log (n/\Delta)}{\eta} \right)$, we have
\[  \vz\su{t} = (1+\eta \theta)^t \vz\su{0} + \err\su{t},  \]
where $\theta = \frac{n(a-b)}{K}$ and $\norm{\err\su{t}} \le O(\log \frac{n}{\Delta}) \frac{c}{\sqrt{n^{\rho+2}}} \norm{\vz\su{t}}$.
\end{lemma}

The lemma is a strong structural statement, saying that $\vz$ essentially only gets scaled in each iteration! 

\begin{proof}
By multiplying the update, Equation (\ref{eq:main-update}), by $\Pi$, we see that
\[ \vz\su{t+1} = \mL \vz\su{t} - \eta \Pi \mE\su{t} \vw\su{t}.  \]
As we have done in the proof of Lemma~\ref{lem:num-iterations}, the norm of the second term can be bounded by $4\eta (\eps \Delta)^3 < \frac{\norm{\vz\su{t}}}{\sqrt{n}}$. 
Now for the first term, observe that the top $(K-1)$ eigenvalues of $\Pi$ are all in the range $(1+\eta \theta - \eta \frac{c}{\sqrt{n^{\rho+2}}}, 1+\eta \theta + \eta \frac{c}{\sqrt{n^{\rho+2}}})$. Thus $\mL$ acts as a scaling of the identity (within the space of the top $(K-1)$ eigenvectors), up to an additive error $\eta \frac{c}{\sqrt{n^{\rho+2}}}$. This implies that
\[ \vz\su{t+1} = (1+\eta \theta) \vz\su{t} + \eta \frac{2c}{\sqrt{n^{\rho+2}}} \xi\su{t},  \]
for some error vector $\xi\su{t}$ with $\norm{\xi\su{t}} \le \norm{\vz\su{t}}$.  Once we have this for all $t$, we can sum up to obtain
\begin{equation}\label{eq:z-update-helper}
    \vz\su{t} = (1+\eta \theta)^t \vz\su{0} + \frac{2c\eta}{\sqrt{n^{\rho+2}}} \cdot \err,
\end{equation}
where 
\[ \err = \left[ (1+\eta \theta)^{t-1} \xi\su{0} + (1+\eta\theta)^{t-2} \xi\su{1} + \dots  \right]. \]
Now we can inductively maintain the property that $\norm{\vz\su{j}} \le 2 (1+\eta \theta)^j \norm{\vz\su{0}}$, to conclude that 
\[ \norm{\err} \le t \cdot (1+\eta\theta)^t \norm{\vz\su{0}}. \]
Noting that $t \le \frac{ \log(n/\Delta)}{\eta}$, we can plug this into Equation (\ref{eq:z-update-helper}) to complete the proof of the lemma.
\end{proof}

This structural theorem is very powerful. To see it, suppose $\vu$ is any vector in $\R^{2n}$, and suppose we originally had 
\[ |\iprod{\vu, \vz\su{0}}| \ge \beta \norm{\vz\su{0}}. \]
Then, we can use Lemma~\ref{lem:z-update-simple} to conclude that
\[ |\iprod{\vu, \vz\su{t_f}}| \ge \beta \norm{\vz\su{t_f}} - \frac{c \log (n/\Delta)}{\sqrt{n^{\rho+2}}} \norm{\vz\su{t_f}}. \]

Thus, if $\beta > \frac{4 c\log \frac{n}{\Delta}}{\sqrt{n^{\rho+2}}}$, we can conclude that 
\[ |\iprod{\vu, \vz\su{t_f}}| \ge \frac{\beta}{2} \norm{\vz\su{t_f}} . \]

This leads us to the following result.

\begin{theorem}\label{thm:main-separation}
Suppose $G$ is drawn from a symmetric SBM with $K$ communities and suppose the co-occurrence matrix $\mC$ used for constructing the embeddings satisfies Assumption~\ref{ass:co-occurrence-error}. Let $\vx$ be the ($x$ component of the) embedding obtained by Algorithm \ref{alg:gradient_descent}. Then with probability $\ge 0.9$ over the initialization, the embedding satisfies:
\[ \forall \text{ clusters } i \ne j,~ | \mu_{\sV_i} - \mu_{\sV_j} | \ge \frac{\eps \Delta}{20 K^2 \sqrt{n}}. \] (As before, $\mu_{\sV_i} := \frac{\sum_{r \in \sV_i} x_r}{|\sV_i|}$).
\end{theorem}

\paragraph{Remark.} It is natural to ask if this separation is sufficient. For this, we do the following heuristic calculation. In the end, we have $\norm{x} =\eps \Delta$. This means all the coordinates are roughly of magnitude $\frac{\eps \Delta}{\sqrt{n}}$ (say they are in the range $[-2\frac{\eps \Delta}{\sqrt{n}}, 2\frac{\eps \Delta}{\sqrt{n}}] $. The theorem says that the typical gap between the cluster centers is a constant factor of this interval length (assuming $K$ is a constant), with high probability. 

The main idea of the proof is to show that a sufficient amount of ``initial separation'' between cluster means must exist because the initialization is random. Then, the discussion preceding the theorem can be used to show that this separation must persist as we update. The key point is that as a fraction of the length $\vz\su{0}$, the initial separation is significant. The proof details are deferred to Appendix~\ref{proof:thm:main-separation}.

Theorems \ref{thm:cluster-inside} and \ref{thm:main-separation} lead to the following approximate recovery result.

\begin{theorem}\label{thm:weak_recovery}
Suppose $G$ is drawn from a symmetric SBM with $K$ communities and suppose the co-occurrence matrix $\mC$ used for constructing the embeddings satisfies Assumption~\ref{ass:co-occurrence-error} and $\rho \in \left(\frac{1}{3}, 1 \right)$. Suppose we run Algorithm \ref{alg:gradient_descent} for $t_f < \frac{4\log{\left(n/\Delta\right)}}{\eta}$ iterations. Then with high probability, we recover $1 - o(1)$ fraction of each community.
\end{theorem}

\begin{proof}
From Theorem \ref{thm:cluster-inside}, after $t_f$ iterations, we have $$\left\lVert \vx - \mathbf{\mu} \right\rVert^2 = \sum_{i \in [K]} \sum_{j \in V_i} | x_j - \mu_i |^2 \le \frac{c_1\left\lVert \vx \right\rVert^2}{\Delta^2}.$$ for a positive constant $c_1$. Furthermore, from Theorem \ref{thm:main-separation} we have for all pairs of clusters $i \ne j$, $$\left\lvert \mu_i - \mu_j \right\rvert^2 \ge \frac{\left( \epsilon\Delta \right)^2 }{c_2 K^4 n}$$ for a positive constant $c_2$.

The result follows from a simple application of Markov's inequality: \begin{equation}\label{eq:markov}\Pr{\left(\lvert x_j - \mu_i \rvert^2 \ge \frac{(\epsilon \Delta)^2}{4 c_2 K^4 n} \right)} \le \frac{4 c_2 K^4 n \mathbb{E}[|x_j - \mu_i |^2]}{(\epsilon \Delta)^2}.\end{equation}

Each cluster has size $n/K$. This, along with Theorem \ref{thm:cluster-inside}, implies that within a cluster $V_i$ we have $\mathbb{E}[\lvert x_j - \mu_i \rvert^2] \le \frac{K}{n}\frac{2c_1(\epsilon \Delta)^2}{n^{1/3}}$. (Recall that the algorithm terminates after $t_f$ iterations when $\epsilon \Delta \le \lVert \vw^{t_f} \rVert \le 2\epsilon \Delta $.) Plugging this into Equation $(\ref{eq:markov})$ we have

\begin{equation}
    \Pr{\left(\lvert x_j - \mu_i \rvert^2 \ge \frac{(\epsilon \Delta)^2}{4 c_2 K^4 n} \right)} \le \frac{8 c_1 c_2 K^5 }{n^{1/3}}.    
\end{equation}

This implies that the number of vertices in a cluster $V_i$ that are not within a distance of $\frac{1}{2} \cdot \frac{\epsilon \Delta}{(c_2 K^4 n)^{1/2}}$ of their cluster mean (or half the distance between cluster means) is no greater than $\frac{8 c_1 c_2 K^5}{n^{1/3}}$. (We're also assuming that the number of clusters $K$ is constant.) This concludes the proof.
\end{proof}

\section{Experiments}
\label{experiments}
This section presents experimental results that support our theoretical analysis. Section \ref{sec:calc_emb} shows that DeepWalk embeddings can completely recover the community structure in a graph even when the embedding dimension is lower than the number of communities. Section \ref{sec:lin_approx} shows that embeddings trained using a linear approximation to the update equation remains close to embeddings trained using the nonlinear update.

\vspace{-6pt}
\subsection{Calculated Embeddings}\label{sec:calc_emb}

\begin{figure*}[tp]
    \centering
    \begin{subfigure}
        \centering
        \includegraphics[scale=0.37]{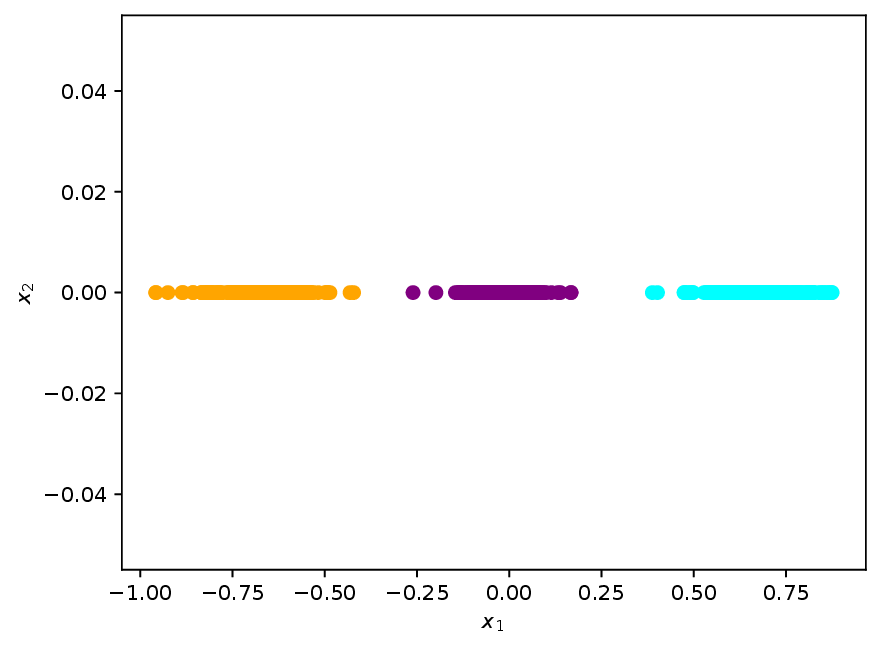}
    \end{subfigure}
    \begin{subfigure}
        \centering
        \includegraphics[scale=0.37]{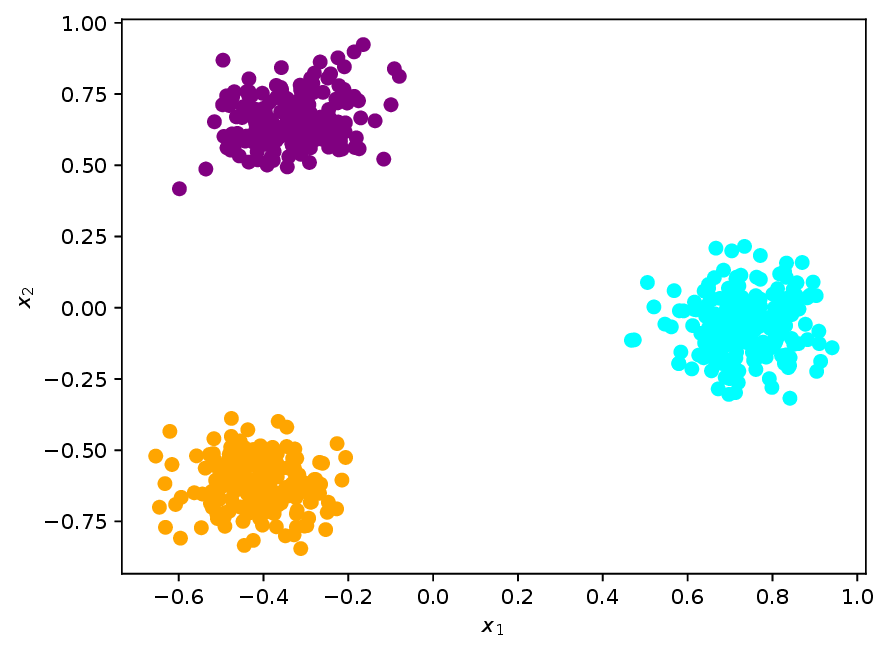}
    \end{subfigure}
    \begin{subfigure}
        \centering
        \includegraphics[scale=0.37]{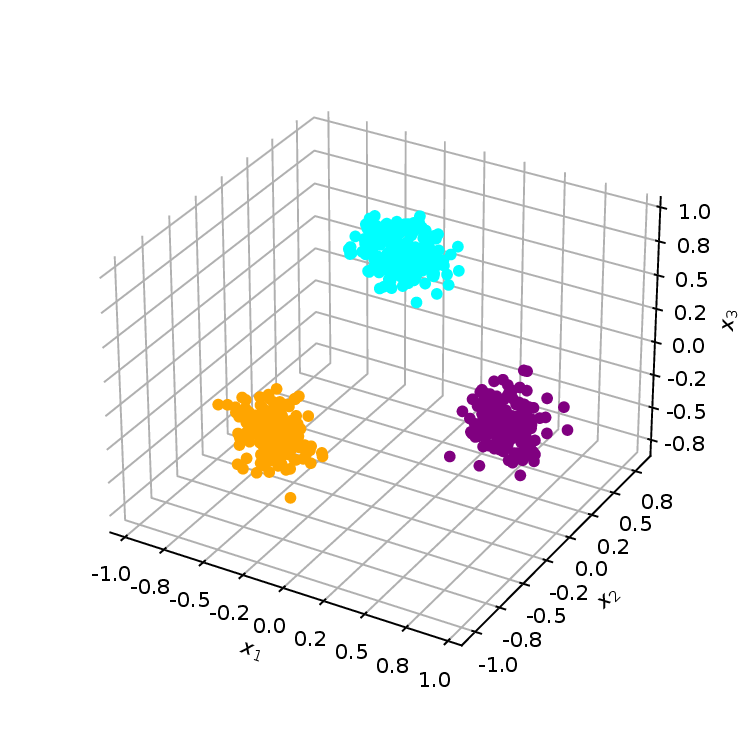}
    \end{subfigure}
    \caption{For a random graph drawn from a stochastic block model with $K=3$ blocks and $n=600$, we show the calculated (left) 1-dimensional (middle) 2-dimensional and (right) 3-dimensional embeddings.}
    \label{fig:embeddings}
\end{figure*}

This section presents $1$-dimensional, $2$-dimensional, and $3$-dimensional embeddings produced by training the deepwalk algorithm. A graph was drawn from a stochastic block model with $K=3$ communities, $n=600$ nodes and parameters $q=0.1$ and $p=4q$. We run the algorithm for $T=100$ iterations and used a learning rate of $\eta = 0.01$. The embeddings are initialized randomly so that $\lVert \vx^{(t)} \rVert_{\infty} \le 0.01$ and $\lVert \vy^{(t)} \rVert_{\infty} \le 0.01$.  

All three sets of embeddings completely recover the community structure in the graph (see Figure \ref{fig:embeddings}).  As our analysis shows, even 1-dimensional embeddings show a clear separation between the clusters. 

\vspace{-6pt}
\subsection{Linear Approximation to the Update}\label{sec:lin_approx}

We note in our analysis that the linear part of the update dominates the convergence behavior, suggesting that the update equations can be approximated by a linear function. This section provides experimental results that supports this observation even for reasonably large graphs.

Three random graphs are drawn from a stochastic block model with $K=2$ clusters with $n=200, 500, 1000$ nodes and parameters $q=0.1$ and $p=4q$. On each graph, embeddings are calculated using both the original nonlinear gradient descent update equations and the linear approximation to the update equation. The distance $d(\vx^{(t)}, \mathbf{\ell}^{(t)}) = \lVert \vx^{(t)} - \mathbf{\ell}^{(t)} \rVert$ between an embedding trained with a nonlinear update $\vx^{(t)}$ and an embedding trained with a linear update $\mathbf{\ell}^{(t)}$ is tracked throughout the training procedure (see Figure~\ref{fig:close_to_linear}). The embeddings were initialized randomly so that $\lVert \vx^{(0)} \rVert_{\infty} = \lVert \mathbf{\ell}^{(0)} \rVert_{\infty} \le \frac{1}{\sqrt{n}}$ and $\vw^{(0)} = \mathbf{\ell}^{(0)}$. The learning rate was set for $\eta = \frac{1}{n}$ and training was rate for $T=75$ iterations.

The distance $d(\vx^{(t)}, \vw^{(t)})$ remains nearly zero for the first several iterations, allowing ample time for the clusters to separate, despite the graphs being a reasonable size. The results also suggests that the linear approximation may hold for a larger parameter regime than the ones we use in our analysis. Determining the right trade-offs between the radius of initialization $\eps$, the learning rate $\eta$, and the SBM parameters $K, p, q$, is an interesting open question.

\begin{figure}
    \centering
    \includegraphics[scale=0.5]{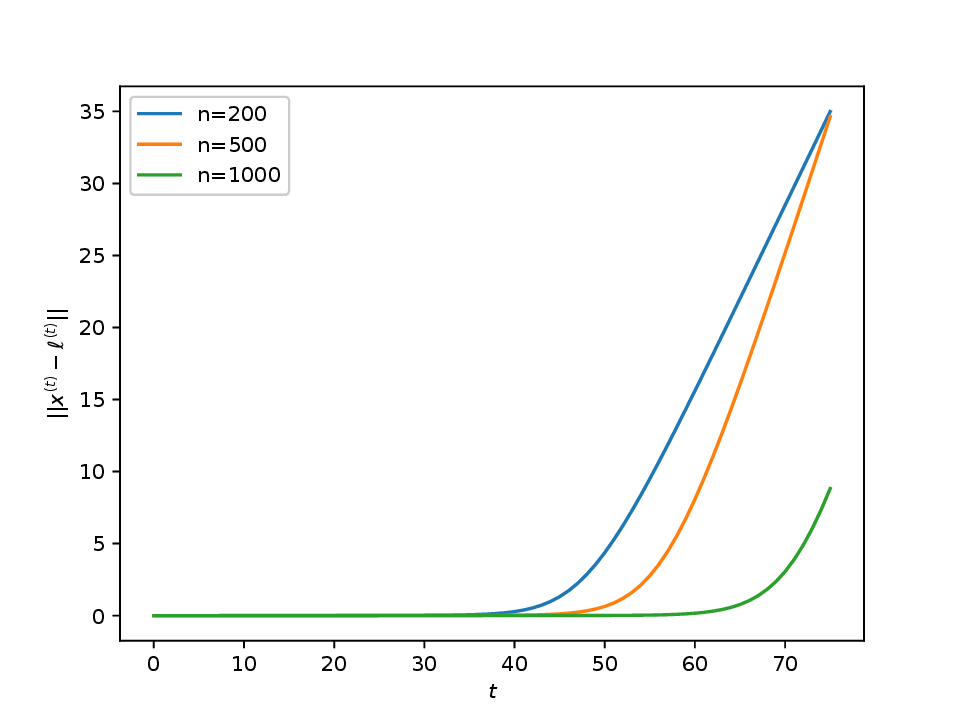}
    \caption{For random graphs drawn from a stochastic block model with $K=2$ blocks and $n=200, 500, 1000$ nodes, we track the distance between the original nonlinear gradient descent update for deepwalk and its linear approximation.}
    \label{fig:close_to_linear}
\end{figure}

\section{Conclusion}
We give the first provable convergence guarantees, that we are aware of, for the DeepWalk algorithm and obtaining low-dimensional embeddings of vertices of a graph. We show that for a graph drawn from a stochastic block model (SBM), the DeepWalk embeddings obtained by a random initialization in a small enough ball and trained with gradient descent, recovers the community structure with high probability. Unlike previous works, we analyze the gradient descent dynamics directly; the key technical tool is to show that when embeddings are of small enough length, the gradient descent update can be approximated by a linear update up to a small error. Our approximation technique may be applicable to other related graph embedding problems, which gives an interesting avenue for future work. 

\section*{Acknowledgements}
The authors are supported by the National Science Foundation under Grant Nos. CCF-2008688 and CCF-2047288.

\bibliography{references}

\begin{thebibliography}{38}
\providecommand{\natexlab}[1]{#1}
\providecommand{\url}[1]{\texttt{#1}}
\expandafter\ifx\csname urlstyle\endcsname\relax
  \providecommand{\doi}[1]{doi: #1}\else
  \providecommand{\doi}{doi: \begingroup \urlstyle{rm}\Url}\fi

\bibitem[Abbe(2017)]{abbe2017community}
E.~Abbe.
\newblock Community detection and stochastic block models: recent developments.
\newblock \emph{The Journal of Machine Learning Research}, 18\penalty0 (1):\penalty0 6446--6531, 2017.

\bibitem[Alon(1986)]{alon}
N.~Alon.
\newblock Eigenvalues and expanders.
\newblock \emph{Combinatorica}, 6\penalty0 (2):\penalty0 83--96, June 1986.

\bibitem[Arora et~al.(2009)Arora, Rao, and Vazirani]{arora-rao-vazirani}
S.~Arora, S.~Rao, and U.~Vazirani.
\newblock Expander flows, geometric embeddings and graph partitioning.
\newblock \emph{J. ACM}, 56\penalty0 (2), apr 2009.

\bibitem[Arora et~al.(2018)Arora, Hu, and Kothari]{arora2018analysis}
S.~Arora, W.~Hu, and P.~K. Kothari.
\newblock An analysis of the t-sne algorithm for data visualization.
\newblock In \emph{Conference On Learning Theory}, pages 1455--1462. PMLR, 2018.

\bibitem[Backstrom and Leskovec(2011)]{backstrom2011links}
L.~Backstrom and J.~Leskovec.
\newblock Supervised random walks: Predicting and recommending links in social networks.
\newblock In \emph{Proceedings of the Fourth ACM International Conference on Web Search and Data Mining}, WSDM '11, page 635–644, New York, NY, USA, 2011.

\bibitem[Barot et~al.(2021)Barot, Bhamidi, and Dhara]{barot2021community}
A.~Barot, S.~Bhamidi, and S.~Dhara.
\newblock Community detection using low-dimensional network embedding algorithms.
\newblock \emph{arXiv preprint arXiv:2111.05267}, 2021.

\bibitem[Belkin and Niyogi(2001)]{belkin2001eigenmaps}
M.~Belkin and P.~Niyogi.
\newblock Laplacian eigenmaps and spectral techniques for embedding and clustering.
\newblock In \emph{Proceedings of the 14th International Conference on Neural Information Processing Systems: Natural and Synthetic}, NIPS'01, page 585–591, Cambridge, MA, USA, 2001.

\bibitem[Cai and Ma(2022)]{cai2022theoretical}
T.~T. Cai and R.~Ma.
\newblock Theoretical foundations of t-sne for visualizing high-dimensional clustered data, 2022.

\bibitem[Chanpuriya and Musco(2020)]{musco2020infinitywalk}
S.~Chanpuriya and C.~Musco.
\newblock Infinitewalk: Deep network embeddings as laplacian embeddings with a nonlinearity.
\newblock \emph{CoRR}, abs/2006.00094, 2020.

\bibitem[Davison et~al.(2023)Davison, Morgan, and Ward]{davison2023community}
A.~Davison, S.~C. Morgan, and O.~G. Ward.
\newblock Community detection and classification guarantees using embeddings learned by node2vec.
\newblock \emph{arXiv preprint arXiv:2310.17712}, 2023.

\bibitem[Grover and Leskovec(2016)]{grover2016node2vec}
A.~Grover and J.~Leskovec.
\newblock node2vec: Scalable feature learning for networks.
\newblock In \emph{Proceedings of the 22nd ACM SIGKDD international conference on Knowledge discovery and data mining}, pages 855--864, 2016.

\bibitem[Hall(1970)]{hall}
K.~M. Hall.
\newblock An r-dimensional quadratic placement algorithm.
\newblock \emph{Manage. Sci.}, 17\penalty0 (3):\penalty0 219–229, nov 1970.

\bibitem[Hamilton et~al.(2017)Hamilton, Ying, and Leskovec]{hamilton2017graphsage}
W.~Hamilton, Z.~Ying, and J.~Leskovec.
\newblock Inductive representation learning on large graphs.
\newblock In \emph{Advances in Neural Information Processing Systems}, volume~30. Curran Associates, Inc., 2017.

\bibitem[Harker and Bhaskara(2023)]{harker2023structure}
C.~Harker and A.~Bhaskara.
\newblock Structure of nonlinear node embeddings in stochastic block models.
\newblock In \emph{Proceedings of The 26th International Conference on Artificial Intelligence and Statistics}, volume 206 of \emph{Proceedings of Machine Learning Research}, pages 6764--6782. PMLR, 25--27 Apr 2023.

\bibitem[Holland et~al.(1983)Holland, Laskey, and Leinhardt]{Holland1983StochasticBM}
P.~Holland, K.~B. Laskey, and S.~Leinhardt.
\newblock Stochastic blockmodels: First steps.
\newblock \emph{Social Networks}, 5:\penalty0 109--137, 1983.

\bibitem[Kloepfer et~al.(2021)Kloepfer, Aviles-Rivero, and Heydecker]{kloepfer2021delving}
D.~Kloepfer, A.~I. Aviles-Rivero, and D.~Heydecker.
\newblock Delving into deep walkers: A convergence analysis of random-walk-based vertex embeddings.
\newblock \emph{arXiv preprint arXiv:2107.10014}, 2021.

\bibitem[Levy and Goldberg(2014)]{levy2014advances}
O.~Levy and Y.~Goldberg.
\newblock Neural word embedding as implicit matrix factorization.
\newblock In \emph{Advances in Neural Information Processing Systems}, volume~27. Curran Associates, Inc., 2014.

\bibitem[Linderman and Steinerberger(2019)]{linderman2019clustering}
G.~C. Linderman and S.~Steinerberger.
\newblock Clustering with t-sne, provably.
\newblock \emph{SIAM Journal on Mathematics of Data Science}, 1\penalty0 (2):\penalty0 313--332, 2019.

\bibitem[Linial et~al.(1994)Linial, London, and Rabinovich]{linial-london-rabinovich}
N.~Linial, E.~London, and Y.~Rabinovich.
\newblock The geometry of graphs and some of its algorithmic applications.
\newblock In \emph{Proceedings 35th Annual Symposium on Foundations of Computer Science}, pages 577--591, 1994.

\bibitem[McSherry(2001)]{mcsherry}
F.~McSherry.
\newblock Spectral partitioning of random graphs.
\newblock In \emph{Proceedings 42nd IEEE Symposium on Foundations of Computer Science}, pages 529--537, 2001.

\bibitem[Mikolov et~al.(2013{\natexlab{a}})Mikolov, Chen, Corrado, and Dean]{mikolov2013efficient}
T.~Mikolov, K.~Chen, G.~Corrado, and J.~Dean.
\newblock Efficient estimation of word representations in vector space.
\newblock \emph{arXiv preprint arXiv:1301.3781}, 2013{\natexlab{a}}.

\bibitem[Mikolov et~al.(2013{\natexlab{b}})Mikolov, Sutskever, Chen, Corrado, and Dean]{mikolov2013distributed}
T.~Mikolov, I.~Sutskever, K.~Chen, G.~S. Corrado, and J.~Dean.
\newblock Distributed representations of words and phrases and their compositionality.
\newblock \emph{Advances in neural information processing systems}, 26, 2013{\natexlab{b}}.

\bibitem[Ng et~al.(2001)Ng, Jordan, and Weiss]{Ng2001OnSC}
A.~Ng, M.~I. Jordan, and Y.~Weiss.
\newblock On spectral clustering: Analysis and an algorithm.
\newblock In \emph{NIPS}, 2001.

\bibitem[Perozzi et~al.(2014)Perozzi, Al-Rfou, and Skiena]{perozzi2014deepwalk}
B.~Perozzi, R.~Al-Rfou, and S.~Skiena.
\newblock Deepwalk: Online learning of social representations.
\newblock In \emph{Proceedings of the 20th ACM SIGKDD international conference on Knowledge discovery and data mining}, pages 701--710, 2014.

\bibitem[Qiu et~al.(2018)Qiu, Dong, Ma, Li, Wang, and Tang]{qiu2018network}
J.~Qiu, Y.~Dong, H.~Ma, J.~Li, K.~Wang, and J.~Tang.
\newblock Network embedding as matrix factorization: Unifying deepwalk, line, pte, and node2vec.
\newblock In \emph{Proceedings of the eleventh ACM international conference on web search and data mining}, pages 459--467, 2018.

\bibitem[Qiu et~al.(2020)Qiu, Wang, Liao, Peng, and Tang]{qiu2020matrix}
J.~Qiu, C.~Wang, B.~Liao, R.~Peng, and J.~Tang.
\newblock A matrix chernoff bound for markov chains and its application to co-occurrence matrices.
\newblock \emph{Advances in Neural Information Processing Systems}, 33:\penalty0 18421--18432, 2020.

\bibitem[Rohe et~al.(2011)Rohe, Chatterjee, and Yu]{rohe2011spectral}
K.~Rohe, S.~Chatterjee, and B.~Yu.
\newblock Spectral clustering and the high-dimensional stochastic blockmodel.
\newblock \emph{The Annals of Statistics}, 39\penalty0 (4):\penalty0 1878--1915, 2011.

\bibitem[Satpathi and Srikant(2021)]{satpathi2021dynamics}
S.~Satpathi and R.~Srikant.
\newblock The dynamics of gradient descent for overparametrized neural networks.
\newblock In \emph{Learning for Dynamics and Control}, pages 373--384. PMLR, 2021.

\bibitem[Sinclair and Jerrum(1989)]{jerrum-sinclair}
A.~Sinclair and M.~Jerrum.
\newblock Approximate counting, uniform generation and rapidly mixing markov chains.
\newblock \emph{Information and Computation}, 82\penalty0 (1):\penalty0 93--133, 1989.

\bibitem[Spielman and Teng(2007)]{spielman-teng-spectral-partitioning}
D.~A. Spielman and S.-H. Teng.
\newblock Spectral partitioning works: Planar graphs and finite element meshes.
\newblock \emph{Linear Algebra and its Applications}, 421\penalty0 (2):\penalty0 284--305, 2007.

\bibitem[Stewart and Sun(1990)]{stewart90}
G.~Stewart and J.~Sun.
\newblock \emph{Matrix Perturbation Theory}.
\newblock Computer Science and Scientific Computing. Elsevier Science, 1990.

\bibitem[St{\"o}ger and Soltanolkotabi(2021)]{stoger2021small}
D.~St{\"o}ger and M.~Soltanolkotabi.
\newblock Small random initialization is akin to spectral learning: Optimization and generalization guarantees for overparameterized low-rank matrix reconstruction.
\newblock \emph{Advances in Neural Information Processing Systems}, 34:\penalty0 23831--23843, 2021.

\bibitem[Tang et~al.(2015)Tang, Qu, Wang, Zhang, Yan, and Mei]{tang2015line}
J.~Tang, M.~Qu, M.~Wang, M.~Zhang, J.~Yan, and Q.~Mei.
\newblock Line: Large-scale information network embedding.
\newblock In \emph{Proceedings of the 24th International Conference on World Wide Web}, WWW '15, page 1067–1077, Republic and Canton of Geneva, CHE, 2015.

\bibitem[Vershynin(2018)]{vershynin2018high}
R.~Vershynin.
\newblock \emph{High-dimensional probability: An introduction with applications in data science}, volume~47.
\newblock Cambridge university press, 2018.

\bibitem[Von~Luxburg(2007)]{von2007tutorial}
U.~Von~Luxburg.
\newblock A tutorial on spectral clustering.
\newblock \emph{Statistics and computing}, 17\penalty0 (4):\penalty0 395--416, 2007.

\bibitem[Vu(2005)]{vu2005spectralnorm}
V.~H. Vu.
\newblock Spectral norm of random matrices.
\newblock In \emph{Proceedings of the Thirty-Seventh Annual ACM Symposium on Theory of Computing}, STOC '05, page 423–430, New York, NY, USA, 2005.

\bibitem[Wang et~al.(2017)Wang, Cui, Wang, Pei, Zhu, and Yang]{wang2017community}
X.~Wang, P.~Cui, J.~Wang, J.~Pei, W.~Zhu, and S.~Yang.
\newblock Community preserving network embedding.
\newblock In \emph{Proceedings of the Thirty-First AAAI Conference on Artificial Intelligence}, AAAI'17, page 203–209. AAAI Press, 2017.

\bibitem[Zhang and Tang(2021)]{zhang2021consistency}
Y.~Zhang and M.~Tang.
\newblock Consistency of random-walk based network embedding algorithms.
\newblock \emph{arXiv preprint arXiv:2101.07354}, 2021.

\end{thebibliography}
\bibliographystyle{abbrvnat}

\newpage
\appendix
\onecolumn

\section{Proofs From Section \ref{sec:structure_algo}}
\subsection{Properties of Co-occurrence Matrix}\label{sec:cooccurrence_prop}

\begin{lemma}\label{lem:cooccurrence_prop}
    Suppose $\overline{\mA}$ is the expected adjacency matrix of a graph $G$ drawn from an SBM with $K$ communities and parameters $p,q$ as described in Section~\ref{sec:background-sbm}. Also, suppose that $\overline{\mC}$ is the co-occurrence matrix constructed as in Lemma~\ref{lem:co-occur-ij} from the expected adjacency matrix $\overline{\mA}$ with random walks of length $L$ and a window size of $T$ such that $L > T$.  Let $\alpha_1 \ge \alpha_2 \ge ... \ge \alpha_n$ be the eigenvalues of $\overline{\mA}$ and $ \lambda_1 \ge \lambda_2 \ge ... \ge \lambda_n$ be the eigenvalues of $\overline{\mC}$. Then the matrix $\overline{\mC}$ has the following properties:
    \begin{enumerate}
        \item $\overline{\mC}$ has a block structure. In other words, it has the form 
        \begin{equation*}
                \overline{\mC} = \underbrace{\begin{bmatrix} a & b & ... & b \\ b & a & ... & b \\ \vdots & \vdots & \ddots & \vdots \\ b & b & ... & a \end{bmatrix}}_{K \times K} \otimes ~\mJ_{\frac{n}{K} \times \frac{n}{K}}
                \end{equation*} for some parameters a, b.\label{lem:C_property_1}
        \item The eigenvalues of $\overline{\mC}$ are $$\lambda_i = \begin{cases}\frac{2}{n}\left(TL - \frac{T(T+1)}{2}\right) &\text{ if } i=1, \\ \frac{2}{n}\sum_{t=1}^T (L-t) \left(\frac{\alpha_i}{\alpha_1}\right)^t &\text{ if } i=2, \dots, K, \\ 0 &\text{ if } i=K+1, \dots, n, \end{cases}$$ where $\alpha_1 = \frac{n}{K}p + \frac{n(K-1)}{K}q$ and $\alpha_i = \frac{n}{K}(p-q)$ for $i = 2, \dots, K$.\label{lem:C_property_2}
        \item The difference $a-b$ between the on-diagonal block and off-diagonal block entries $a$ and $b$ of $\overline{\mC}$ is $$a-b = \frac{2K}{n^2}\sum_{t=1}^T (L-t) \left(\frac{p-q}{p+(K-1)q}\right)^t.$$ \label{lem:C_property_3}
        \item The on-diagonal block entries $a$ and off-diagonal block entries $b$ of $\overline{\mC}$ are: 
        \begin{align*} 
        a &= \frac{2}{n^2}\left( \left(TL - \frac{T(T+1)}{2}\right) + (K-1)\sum_{t=1}^T (L-t) \left(\frac{p-q}{p+(K-1)q)} \right)^t \right), \\
        b &= \frac{2}{n^2} \left( \left(TL - \frac{T(T+1)}{2}\right) - \sum_{t=1}^T (L-t) \left(\frac{p-q}{p+(K-1)q}\right)^t \right).\end{align*} \label{lem:C_property_4}
    \end{enumerate}
\end{lemma}

\begin{proof}
(\textit{Proof of 1.}) Straightforward from the block structure of $\overline{\mA}$.

(\textit{Proof of 2.}) From Lemma~\ref{lem:co-occur-ij}, we can write $\overline{\mC}$ as

\begin{equation}
    \overline{\mC} = 2\sum_{t=1}^T \frac{(L-t)}{n\overline{d}} \mD_{\overline{\mA}}\left( \overline{\mP}^{t} \right)
\end{equation}

where $\overline{d} = \frac{n}{K}p + \frac{n(K-1)}{K}q$ and $\overline{\mP} = \mD_{\overline{\mA}}^{-1} \overline{\mA}$. Since $\mD_{\overline{\mA}}^{-1}$ is a diagonal matrix with diagonal entries $\frac{1}{\overline{d}}$, the matrix $\overline{\mP}$ is a real-symmetric matrix and we can write it as $$\overline{\mP} = \mV \mathbf{\Sigma} \mV^{\top},$$ where the columns of $\mV$ are the eigenvectors of $\overline{\mP}$ and $\mathbf{\Sigma} = diag(\sigma_1, \sigma_2, ..., \sigma_n)$ with $\sigma_1 \ge \sigma_2 \ge ... \ge \sigma_n$ are the eigenvalues of $\overline{\mP}$. 

Recall that the eigenvalues of $\overline{\mA}$ are $$\alpha_i = \begin{cases} \frac{n}{K}p + \frac{n(K-1)}{K}q &\text{ if }i=1, \\ \frac{n(p-q)}{K} &\text{ if }i=2, \dots, K, \\ 0 &\text{ if }i=K+1, \dots , n. \end{cases}$$

Therefore, since $\overline{d} = \alpha_1$, the eigenvalues of $\overline{\mP}$ are:

\begin{equation}
    \sigma_i = \begin{cases} 1 &\text{ if }i=1, \\ \frac{\alpha_i}{\alpha_1} &\text{ if } i=2, \dots, K \\ 0 &\text{ if }i = K+1, \dots, n. \end{cases}
\end{equation}

We can now write $\overline{\mC}$ as

\begin{align*}
    \overline{\mC} &= \frac{2}{n}\sum_{t=1}^T (L-t) \left(\mV \mathbf{\Sigma}^t \mV^{\top}\right) \\
    &= \mV \left( \frac{2}{n}\sum_{t=1}^T (L-t) \cdot \mathbf{\Sigma}^t \right) \mV^{\top} \\
    &= \mV \mathbf{\Lambda} \mV^{\top}
\end{align*}

where $\mathbf{\Lambda} = diag(\lambda_1, \lambda_2 , ..., \lambda_n)$ with $\lambda_i = \frac{2}{n}\sum_{t=1}^T (L-t)\sigma_i^t$ for all $i = 1, ..., n$. Therefore, the eigenvalues of $\overline{\mC}$ are 
\begin{equation}\label{eq:eigv_C_pq} \lambda_i = \begin{cases} \frac{2}{n}\left(TL - \frac{T(T+1)}{2}\right) &\text{ if } i=1, \\ \frac{2}{n}\sum_{t=1}^T (L-t) \left(\frac{\alpha_i}{\alpha_1}\right)^t &\text{ if }i=2, \dots, K, \\ 0 &\text{ if }i=K+1, \dots, n, \end{cases}\end{equation} where we used $\sum_{t=1}^T (L-t) = TL - \frac{T(T+1)}{2}$.

(\textit{Proof of 3.}) First, notice that we can express the eigenvalues of $\overline{\mC}$ in terms of its on-diagonal entries $a$ and off-diagonal entries $b$:

\begin{equation}\label{eq:eigv_C_ab}
    \lambda_i = \begin{cases} \frac{n}{K}a + \frac{n(K-1)}{K}b &\text{ if }i=1, \\ \frac{n(a-b)}{K} &\text{ if }i=2, \dots, K, \\ 0 &\text{ if }i=K+1, \dots , n.\end{cases}
\end{equation}

From equations (\ref{eq:eigv_C_pq}) and (\ref{eq:eigv_C_ab}), we know that for $i = 2, \dots , K$ we have 

\begin{align*}
    \frac{n}{K}(a-b) &= \frac{2}{n}\sum_{t=1}^T (L-t) \left(\frac{p-q}{p+(K-1)q}\right)^t.
\end{align*}

By multiplying both sides by $\frac{K}{n}$, the result immediately follows. 

(\textit{Proof of 4.}) We prove the result for $b$ first. From Equation (\ref{eq:eigv_C_ab}) we know that 

\begin{align}
    \lambda_1 - \lambda_2 &= \frac{n}{K}a + \frac{n(K-1)}Kq - \frac{n}{K}(a-b) = nb. \label{eq:l1_l2_ab}
\end{align}

Meanwhile, from Equation (\ref{eq:eigv_C_pq}) we know that

\begin{equation}\label{eq:lambda1_lambda2_pq}
    \lambda_1 - \lambda_2 = \frac{2}{n}\left(TL - \frac{T(T+1)}{2}\right) - \frac{2}{n}\sum_{t=1}^T (L-t)\left(\frac{p-q}{p+(K-1)q}\right)^t.
\end{equation}

Combining equations (\ref{eq:l1_l2_ab}) and (\ref{eq:lambda1_lambda2_pq}) and solving for $b$ gives

\begin{equation}\label{eq:b_pq}
b = \frac{2}{n^2} \left( \left(TL - \frac{T(T+1)}{2}\right) - \sum_{t=1}^T (L-t) \left(\frac{p-q}{p+(K-1)q}\right)^t \right).
\end{equation}

Now, we find $a$. From Equation (\ref{eq:b_pq}) we have

\begin{equation}\label{eq:a_minus_b_1}
    a - b = a - \left( \frac{2}{n^2} \left( \left(TL - \frac{T(T+1)}{2}\right) - \sum_{t=1}^T (L-t) \left(\frac{p-q}{p+(K-1)q}\right)^t \right) \right).
\end{equation}

From Property \ref{lem:C_property_3}, we have

\begin{equation}\label{eq:a_minus_b_2}
    a-b = \frac{2K}{n^2}\sum_{t=1}^T (L-t) \left(\frac{p-q}{p+(K-1)q}\right)^t.
\end{equation}

Combining equations (\ref{eq:a_minus_b_1}) and (\ref{eq:a_minus_b_2}) and solving for $a$ gives:

\begin{equation}
    a = \frac{2}{n^2}\left( \left(TL - \frac{T(T+1)}{2}\right) + (K-1)\sum_{t=1}^T (L-t) \left(\frac{p-q}{p+(K-1)q}\right)^t \right).
\end{equation}

This completes the proof.

\end{proof}

\subsection{Proof of Lemma \ref{lem:cooccurrence_conc}}\label{sec:cooccurence_conc}

First, we prove the following lemma:

\begin{lemma}\label{lem:transition_diff}
    Suppose $\mA$ is the adjacency matrix of a graph $G$ that is drawn from a symmetric SBM with $K$ equally sized communities and parameters $p,q$ of at least $n^{\rho-1}$ with $\rho \in (0,1)$. Let $\mP = \mD_{\mA}^{-1}\mA$ be the transition matrix of the graph. Similarly, let $\overline{\mP} = \mD_{\overline{\mA}}^{-1}\overline{\mA}$ be the transition matrix of the expected graph and $\overline{d} = \frac{n}{K}p + \frac{n(K-1)}{K}q$ be the expected degree. Then for any integer $t \ge 1$, we have $$\left\lVert \mP^t - \overline{\mP}^{t} \right\rVert = O\left(\sqrt{\frac{\log{n}}{\overline{d}}}\right)$$ with high probability.
\end{lemma}

\begin{proof}
We prove this by induction. First, we prove the base case. Let $t=1$. Then 
\begin{align*}
    \left\lVert \mP - \overline{\mP} \right\rVert &= \left\lVert \mD_{\mA}^{-1}\mA - \mD_{\overline{\mA}}^{-1}\overline{\mA} \right\rVert \\
    &= \left\lVert \mD_{\mA}^{-1} \left(\mA - \overline{\mA}\right) + \left(\mD_{\mA}^{-1} - \mD_{\overline{\mA}}^{-1}\right) \overline{\mA} \right\rVert  \\
    &\le \underbrace{\left\lVert \mD_{\mA}^{-1}\left(\mA - \overline{\mA}\right) \right\rVert}_{(i)} + \underbrace{\left\lVert \left(\mD_{\mA}^{-1} - \mD_{\overline{\mA}}^{-1} \right) \overline{\mA} \right\rVert}_{(ii)}.
\end{align*}

Given a graph $G$ drawn from a symmetric SBM with parameters $K,p,q$ with $p \ge n^{\rho-1}$ for $\rho \in (0,1)$, we have 
\begin{equation}\label{eq:bernstein_d} 
    \left\lvert d_i - \overline{d} \right\rvert = O\left( \sqrt{\overline{d} \log{n}} \right), \forall i \in [n],
\end{equation}
with probability $1 - n^{-5}$. This follows from standard Bernstein/Chernoff bounds \citep{vershynin2018high}. Also, with probably $1 - n^{-5}$, we have 
\begin{equation}
    \left\lVert \mA - \overline{\mA} \right\rVert = O\left(\sqrt{\overline{d}}\right),
\end{equation}
which follows from standard spectral bounds \citep{vu2005spectralnorm}. Therefore, we can bound (i) by

\begin{align*}
    \left\lVert \mD_{\mA}^{-1} \left( \mA - \overline{\mA} \right) \right\rVert &\le \frac{2}{\overline{d}} \cdot O\left(\sqrt{\overline{d}}\right) = O\left(\frac{1}{\sqrt{\overline{d}}}\right).
\end{align*}

Similarly, we can bound (ii) by

\begin{align*}
    \left\lVert \left(\mD_{\mA}^{-1} - \mD_{\overline{\mA}}^{-1} \right) \overline{\mA} \right\rVert &\le \max_{i} \left\lvert \frac{d_i - \overline{d}}{d_i \overline{d}} \right\rvert \cdot \left\lVert \overline{\mA} \right\rVert \\ 
    &= O\left(\overline{d}^{-3/2} \sqrt{\log{n}}\right) \cdot \overline{d} \\
    &= O\left(\sqrt{\frac{\log{n}}{\overline{d}}}\right).
\end{align*}

Combining the bounds of (i) and (ii), we have $$\left\lVert \mP - \overline{\mP} \right\rVert = O\left(\sqrt{\frac{\log{n}}{\overline{d}}}\right).$$

Now, assume that the statement holds for $t-1$, i.e.,

\begin{equation}
    \left\lVert \mP^{t-1} - \overline{\mP}^{t-1} \right\rVert = O\left(\sqrt{\frac{\log{n}}{\overline{d}}} \right).
\end{equation}

Therefore, by induction we have

\begin{align*}
    \left\lVert \mP^t - \overline{\mP}^t \right\rVert &\le \left\lVert \mP^{t-1}\left(\mP - \overline{\mP} \right) \right\rVert + \left\lVert \left(\mP^{t-1} - \overline{\mP}^{t-1} \right) \overline{\mP} \right\rVert \\
    &= O\left(\sqrt{\frac{\log{n}}{\overline{d}}} \right).
\end{align*}

This concludes the proof.

\end{proof}

Now, we precede to prove Lemma \ref{lem:cooccurrence_conc}. Using Lemma~\ref{lem:co-occur-ij}, we can express the co-occurrence matrix $\mC$ as

\[\mC = 2\sum_{t=1}^{T}\frac{(L-t)}{2|E|} \mD_{\mA} \mP^t, \]

and a corresponding expected co-occurrence matrix $\overline{\mC}$ as

\[\overline{\mC} = 2\sum_{t=1}^{T}\frac{(L-t)}{n\overline{d}} \mD_{\overline{\mA}} \overline{\mP}^t , \]

where $\mP = \mD_{\mA}^{-1} \mA$ and $\overline{\mP} = \mD_{\overline{\mA}}^{-1}\overline{\mA}$. From Property \ref{lem:C_property_4} of Lemma \ref{lem:cooccurrence_prop}, we know that $\overline{\mC}$ has a block structure with entries $a$ and $b$ equal to

\begin{align*} 
a &= \frac{2}{n^2}\left( \left(TL - \frac{T(T+1)}{2}\right) + (K-1)\sum_{t=1}^T (L-t) \left(\frac{p-q}{p+(K-1)q)} \right)^t \right), \\
b &= \frac{2}{n^2} \left( \left(TL - \frac{T(T+1)}{2}\right) - \sum_{t=1}^T (L-t) \left(\frac{p-q}{p+(K-1)q}\right)^t \right).\end{align*}

Now, we can express the difference $\mC - \overline{\mC}$ as

\begin{equation}\label{eq:C_diff}
    \mC - \overline{\mC} = 2 \sum_{t=1}^{T} \left( L-t \right) \left(\frac{1}{2|E|} \mD_{\mA} \mP^{t} - \frac{1}{n\overline{d}} \mD_{\overline{\mA}} \overline{\mP}^{t} \right)
\end{equation}

We proceed by showing that each term in the summation has a spectral norm $\le \frac{\sqrt{\log{n}}}{n\sqrt{\overline{d}}}$. Notice that we can write each term in (\ref{eq:C_diff}) as 

\begin{align}
    \frac{1}{2|E|} \mD_{\mA} \mP^{t} - \frac{1}{n\overline{d}} \mD_{\overline{\mA}} \overline{\mP}^{t} &= \frac{1}{2|E|} \mD_{\mA} \mP^t - \frac{1}{n\overline{d}} \mD_{\mA} \mP^t + \frac{1}{n\overline{d}} \mD_{\mA} \mP^t - \frac{1}{n\overline{d}} \mD_{\overline{\mA}} \overline{\mP}^t \nonumber\\
    &= \underbrace{\left( \frac{n\overline{d} - 2|E|}{2|E|n\overline{d}} \right) \mD_{\mA} \mP^t}_{(i)} + \underbrace{\frac{1}{n\overline{d}} \left( \mD_{\mA} \mP^t - \mD_{\overline{\mA}} \overline{\mP}^t \right)}_{(ii)}.\label{eq:two_parts}
\end{align}

We bound (i) and (ii) in (\ref{eq:two_parts}) separately. Given a graph $G$ drawn from a symmetric SBM with parameters $K,p,q$ with $p \ge n^{\rho-1}$ for $\rho \in (0,1)$, we have 
\begin{equation}\label{eq:bernstein_d} 
\left\lvert d_i - \overline{d} \right\rvert = O\left( \sqrt{\overline{d} \log{n}} \right), \forall i \in [n],
\end{equation}
\begin{equation}\label{eq:bernstein_E} 
\left\lvert |E| - \frac{n\overline{d}}{2} \right\rvert = O\left( \sqrt{n\overline{d}\log{n}} \right)
\end{equation} 
with probability $1 - n^{-5}$. These follow from standard Bernstein/Chernoff bounds \citep{vershynin2018high}. Therefore, with high probability, the norm of the first term (i) in (\ref{eq:two_parts}) can be bounded: 

\begin{align*}
    \left\lVert \left( \frac{n\overline{d} - 2|E|}{2|E|n\overline{d}} \right) \mD_{\mA} \mP^t \right\rVert &\le \left\lvert \frac{n\overline{d} - 2|E|}{2|E|n\overline{d}}\right\rvert \cdot \max_i d_i \cdot \left\lVert \mP^t \right\rVert \\
    &\le O\left( \left(n\overline{d}\right)^{-3/2} \sqrt{\log{n}} \right) \cdot \left( \overline{d} + O\left( \sqrt{\overline{d} \log{n}} \right) \right) \cdot \left( 1 \right) \\
&= O\left( \frac{1}{n \sqrt{\overline{d}}} \right).
\end{align*}

Next, we bound the second term (ii) in (\ref{eq:two_parts}). We can write (ii) as

\begin{equation}\label{eq:two_parts_part_two}
    \frac{1}{n\overline{d}} \left( \mD_{\mA} \mP^t - \mD_{\overline{mA}} \overline{\mP}^t \right) = \underbrace{\frac{1}{n\overline{d}} \left(\mD_{\mA} - \mD_{\overline{\mA}} \right) \mP^t}_{(iii)} + \underbrace{\frac{1}{n\overline{d}} \mD_{\overline{\mA}}\left( \mP^t - \overline{\mP}^t \right)}_{(iv)}.
\end{equation}

Using (\ref{eq:bernstein_d}), the norm of (iii) in (\ref{eq:two_parts_part_two}) can be bounded with high probability:

\begin{align*}
    \left\lVert \frac{1}{n\overline{d}} \left( \mD_{\mA} - \mD_{\overline{\mA}} \right) \mP^t \right\rVert &\le \frac{1}{n\overline{d}} \left\lVert \mD_{\mA} - \mD_{\overline{\mA}} \right\rVert \cdot \left\lVert \mP^t \right\rVert \\
    &= \frac{1}{n\overline{d}} \cdot O\left( \sqrt{\overline{d}\log{n}} \right) \\
    &= O\left(\frac{\sqrt{\log{n}}}{n\sqrt{\overline{d}}}\right).
\end{align*}

Now, using Lemma \ref{lem:transition_diff} and (\ref{eq:bernstein_d}), we bound the norm of (iv) in (\ref{eq:two_parts_part_two}):

\begin{align*}
    \left\lVert \frac{1}{n\overline{d}} \mD_{\overline{\mA}}\left(\mP^t - \overline{\mP}^t \right) \right\rVert &\le \frac{1}{n\overline{d}} \left\lVert \mD_{\overline{\mA}} \right\rVert \cdot \left\lVert \mP^t - \overline{\mP}^t \right\rVert \\
    &= O\left(\frac{\sqrt{\log{n}}}{n\sqrt{\overline{d}}}\right).
\end{align*}

Combining the bounds for (i), (iii), and (iv), gives 

\begin{align*}
    \left\lVert \mC - \overline{\mC} \right\rVert &\le 2\sum_{t=1}^T(L-t) \left\lVert \frac{1}{2|E|}\mD_{\mA} \mP^t - \frac{1}{n\overline{d}}\mD_{\overline{\mA}} \overline{\mP}^t \right\rVert \\
    &=2\left(TL - \frac{T(T+1)}{2}\right) \cdot O\left( \frac{\sqrt{\log{n}}}{n\sqrt{\overline{d}}}\right) \\
    &= \lVert \overline{\mC} \rVert \cdot O\left( \sqrt{\frac{\log{n}}{n^{\rho}}} \right).
\end{align*}

This concludes the proof. 

\section{Proofs From Section~\ref{sec:linear_approx_analysis}}

\subsection{Proof of Proposition~\ref{prop:1}}\label{proof:prop1}

Suppose that without loss of generality that $y_1 \le y_2 \le ... \le y_n$. We have

\begin{align*}
    \lVert \mQ - \frac{1}{n}\mJ \rVert_F^2 &= \sum_{i=1}^n \sum_{j=1}^n \left(\frac{e^{x_i y_j}}{\sum_{k=1}^n e^{x_i y_k}} - \frac{1}{n} \right)^2 \\
    &= \sum_{x_i \le 0} \sum_{j=1}^n \left(\frac{e^{x_i (y_j - y_n)}}{\sum_{k=1}^n e^{x_i (y_k - y_n)}} - \frac{1}{n} \right)^2 + \sum_{x_i > 0} \sum_{j=1}^n \left(\frac{e^{x_i (y_j - y_1)}}{\sum_{k=1}^n e^{x_i (y_k - y_1)}} - \frac{1}{n} \right)^2 \\
    &\le \sum_{x_i \le 0} \frac{n}{4} \left( \frac{e^{x_i (y_1 - y_n)} - 1}{\sum_{k=1}^n e^{x_i (y_k - y_n)} } \right)^2 + \sum_{x_i > 0} \frac{n}{4} \left( \frac{e^{x_i (y_n - y_1)} - 1}{\sum_{k=1}^n e^{x_i (y_k - y_1)} } \right)^2 \\
    &\le \frac{1}{4} \left(e^{2\epsilon^2} - 1 \right)^2 \\
    &\le \epsilon^4.
\end{align*}

\noindent The first inequality is an application of Popoviciu's inequality (Theorem \ref{thm:popoviciu}) and in the last step we use the fact that $e^x - 1 \le xe^x$ for $x > 0$ and $xe^x \le 2x$ when $x < \log{2}$. Taking the square root gives the result.

\subsection{Proof of Lemma~\ref{lem:error-norm}}\label{proof:error-norm}

Note that for a block matrix $\mY = \begin{bmatrix} \mathbf{0} & \mX \\ \mX^{\top} & \mathbf{0} \end{bmatrix}$ we have $\lVert \mY \rVert^2_F = 2\lVert \mX \rVert^2_F$. This means that

\begin{align*}
    \lVert \mE^{(t)} \rVert^2_F &= 2\lVert \mD_{\mC}\left(\mQ^{(t)} - \frac{\mJ}{n} \right) \rVert^2_F \\
    &\le 2 \lVert \mD_{\mC} \rVert^2 \lVert \mQ^{(t)} - \frac{1}{n}{\mJ} \rVert^2_F \\
    &\le 2 \lVert \mD_{\mC} \rVert^2 \cdot \epsilon^4,
\end{align*}

where we used proposition~\ref{prop:1} in the last inequality. 

Now we need to bound $\lVert \mD_{\mC} \rVert$. We have

\begin{align*}
\lVert \mD_{\mC} \rVert = \lVert \mC \rVert_{\infty} \le \lVert \mC \rVert \le \frac{n}{K}(a + (K-1)b) + \lVert \mR \rVert \le 1 + \frac{c}{\sqrt{n^{\rho+2}}} \le 2,
\end{align*}

where we have used the fact that $\frac{n}{K}(a + (K-1)b) = O\left(\frac{1}{n}\right)$ (see Lemma \ref{lem:cooccurrence_prop}) and $\lVert \overline{\mC} - \mC \rVert = \lVert \mR \rVert \le \frac{c}{\sqrt{n^{\rho+2}}}$. Therefore $$\lVert \mE^{(t)} \rVert_F \le \sqrt{8}\epsilon^2 < 4 \epsilon^2.$$ This concludes the proof.

\subsection{Proof of Lemma~\ref{lem:cL-properties}}\label{sec:cL-properties}

First, note that for a real symmetric matrix $\mX$ with eigenvalues $\sigma_1, ..., \sigma_n$, the matrix $\mY = \begin{bmatrix} \mathbf{0} & \mX \\ \mX & \mathbf{0} \end{bmatrix}$ has eigenvalues $\alpha_i = \pm \sigma_i$. 

Let $$\overline{\mL} = \begin{bmatrix} \mI & -\eta \overline{\mM} \\ -\eta\overline{\mM} & \mI \end{bmatrix}.$$ Recall from section \ref{sec:structure_algo}, that $\overline{\mM}$ has $(K-1)$ nonzero eigenvalues of $\frac{n}{K}(b-a)$ and $(n-K+1)$ eigenvalues of $0$. This implies that $\overline{\mL}$ has eigenvalues $$\overline{\alpha}_i = \begin{cases} 1 + \eta \frac{n}{K}(a-b) &\text{ if } i=1,\dots,K-1, \\1 &\text{ if } i=K, \dots, 2n - K, \\1 - \eta \frac{n}{K}(a-b) &\text{ if } i = 2n - K + 1, \dots, 2n. \end{cases}$$

For an eigenvalue $\alpha_i$ of $\mL$, Weyl's theorem (Theorem \ref{thm:weyl}) implies that for $i = 1, \dots , K-1$, we have $|\alpha_i - \overline{\alpha}_i| \le \lVert \mR \rVert$, where $\mR = \mL - \overline{\mL}$. This means that for $i = 1, ..., K-1$, we have $$\alpha_i \ge 1 + \eta \frac{n}{K}(a-b) - \eta\frac{c}{\sqrt{n^{\rho+2}}} \ge 1 + \eta \frac{n(a-b)}{2K} = 1 + \eta \gamma.$$

For $i = K, ..., 2n - K$, we have $$\alpha_i \le 1 + \eta\frac{c}{\sqrt{n^{\rho+2}}},$$ and for $i = 2n-K+1 , .., 2n$, we have $$\alpha_i \le 1 - \eta\frac{n}{K}(a-b) + \eta\frac{c}{\sqrt{n^{\rho+2}}} \le 1 + \eta\frac{c}{\sqrt{n^{\rho+2}}}.$$

Lastly, it is now easy to see that the largest eigenvalues have an upper bound of $\alpha_i \le 1 + \eta\frac{n}{K}(a-b) + \eta\frac{c}{\sqrt{n^{\rho+2}}} \le 1 + 2\eta\frac{n}{K}(a-b) = 1 + 4\eta\gamma$. Similarly, the smallest eigenvalues have a lower bound of $\alpha_i \ge 1 - \eta\frac{n}{K}(a-b) - \eta\frac{c}{\sqrt{n^{\rho+2}}} \ge 1 - 2\eta\frac{n}{K}(a-b) = 1 - 4\eta\gamma$. 

This concludes the proof.

\section{Proofs From section \ref{sec:convergence_analysis}}
\subsection{Proof of Theorem \ref{thm:cluster-inside}}\label{proof:cluster-inside}

From the properties of $\overline{\mC}$ described in Section \ref{analysis}, the top $K$ eigenvectors of $\overline{\mC}$ lie in the space spanned by the vectors $\vone_{\sV_1}, \vone_{\sV_2}, \dots, \vone_{\sV_K}$. In the span of these vectors, suppose we wish to find a vector $\vx'$ that minimizes $\norm{\vx - \vx'}^2$. Writing a general vector in the span as $\vx' = \sum_i \alpha_i \vone_{\sV_i}$, we see that in order to minimize $\norm{\vx - \vx'}^2$, we must set $\alpha_i = \mu_{\sV_i}$. Thus, if $\Gamma$ is the projection matrix onto the span of $\{ \vone_{\sV_i}\}_{i=1}^K$, we have 
\[ \norm{\vx - \pmb{\mu}} = \norm{(I - \Gamma) \vx}.  \]

Let us now compare $(I - \Gamma)\vx$ and $(I - \overline{\Pi}) \vw\su{T}$. The two main differences are the following: first, $\overline{\Pi}$ is a $2n \times 2n$ projection matrix (as opposed to $\Gamma$, which is $n \times n$). Second, $\overline{\Pi}$ is a projection onto a $(K-1)$ dimensional subspace (as opposed to $\Gamma$, which projects to $K$ dimensions). The structure of the eigenvectors of $\overline{\mL}$ (see the proof of Lemma~\ref{lem:cL-properties}) implies that the projection $(I -  \overline{\Pi})$ is equivalent to applying an appropriate projection to the first $n$ and the second $n$ coordinates separately. Thus, the term $(I - \overline{\Pi}) \vw\su{t_f}$ includes the projection error for both $\vx$ and $\vy$. Second, since the error in projecting to a $K$ dimensional space is only smaller than the error in projecting to a $(K-1)$ dimensional subspace of it, we get:
\[ \norm{ (I - \Gamma)\vx} \le \norm{(I - \overline{\Pi}) \vw\su{t_f}}.  \]

Next, we try to use the conclusion of Lemma~\ref{lem:error-w-z}. For this, we need to relate $\Pi$ and $\overline{\Pi}$. But this turns out to be easy in our case. Recall that $\Pi$ and $\overline{\Pi}$ are, respectively, the projections onto the span of the top $(K-1)$ eigenvectors of $\mM$ and $\overline{\mM}$ respectively. Since the gap between the $(K-1)$ and $K$th largest eigenvalues for $\mM$ (and also $\overline{\mM}$) is $\ge \frac{n(a-b)}{2K}$ (which is $\Omega\left(\frac{1}{n}\right)$), we can use the classic Davis-Kahan Sin-Theta theorem~\citep{stewart90}, applied to the spectral norm, to obtain:
\[ \norm{\Pi - \overline{\Pi}} \le \frac{2K}{n(a-b)} \norm{\mR},  \]
where $\norm{\mR} \le \frac{c}{\sqrt{n^{\rho + 2}}}$ as we saw before. This implies that
\[ \norm{ (I - \Gamma)\vx} \le \norm{(I - \Pi) \vw\su{t_f}} + O(\frac{1}{\sqrt{n^{\rho}}}) \norm{\vw\su{t_f}}. \]
Combining this with Lemma~\ref{lem:error-w-z} and noting that $\frac{1}{\Delta}$ dominates $\frac{1}{\sqrt{n^{\rho}}}$, the theorem follows.

\subsection{Proof of Theorem~\ref{thm:main-separation}}\label{proof:thm:main-separation}

Fix some $i,j \in [K]$ with $i\ne j$. Consider the $2n$ dimensional vector $\vu$ obtained by taking $\frac{\vone_{\sV_i}}{\sqrt{|\sV_i|}} - \frac{\vone_{\sV_j}}{\sqrt{|\sV_j|}}$ and appending $0$ for the remaining $n$ coordinates. Since the partitions $\sV_i$ are all of size $n/K$, this is simply the vector $\sqrt{\frac{K}{n}} (\vone_{\sV_i} - \vone_{\sV_j})$ with $0$s appended. Note that by construction, $\norm{\vu}^2 = 2$.

Now, since the initial vector $\vw\su{0}$ was chosen uniformly at random, its projection $\vz\su{0}$ is a random vector in the space spanned by the top $(K-1)$ eigenvectors of $\mL$. Since $u$ is a vector also in this span, and since it has constant length, we expect $u$ to have an inner product of roughly $\norm{\vz\su{0}}/\sqrt{(K-1)} $ with $\vz\su{0}$. In fact, since the inner product is distributed as a Gaussian, we have that
\[ \Pr [ |\iprod{\vu, \vz\su{0}}| < \frac{1}{10 K^2 \sqrt{K}} \norm{\vz\su{0}} ] < \frac{1}{10 K^2}. \]

Now, from the discussion preceding the theorem, we have that with probability $1 - \frac{1}{10 K^2}$,
\[  |\iprod{\vu, \vz\su{t_f}}| \ge \frac{1}{20 K^2 \sqrt{K}} \norm{\vz\su{t_f}}. \]
By definition, note that $| \mu_{\sV_i} - \mu_{\sV_j} | = \sqrt{\frac{K}{n}} |\iprod{\vu, \vz\su{t_f}}|$. This implies that with probability $1- \frac{1}{10K^2}$, we have (for fixed $i,j$),
\[ | \mu_{\sV_i} - \mu_{\sV_j} | \ge \frac{1}{20 K^2 \sqrt{n}} \norm{\vz\su{t_f}}.  \]
Taking a union bound over all $i,j$ and using the value of $\norm{\vz\su{t_f}}$, the theorem follows.

\section{Auxiliary Lemmas}

\begin{theorem}[Weyl's Theorem]\label{thm:weyl}
Let $\overline{\mA}$ and $\mE$ be $n \times n$ symmetric matrices. Let $\overline{\lambda}_1 \ge ... \ge \overline{\lambda}_n$ be the eigenvalues of $\overline{\mA}$ and $\lambda_1 \ge ... \ge \lambda_n$ be the eigenvalues of $\mA = \overline{\mA} + \mE$. Then $|\lambda_1 - \overline{\lambda}_1 | \le \lVert \mE \rVert$.
\end{theorem}

\begin{theorem}[Popoviciu's Inequality]\label{thm:popoviciu}
Let $M$ and $m$ be the upper and lower bounds of the entries of a probability vector $\vp = (p_1, ..., p_n).$ Then 

\begin{equation}
    \frac{1}{n} \sum_{i=1}^n (p_i - \frac{1}{n})^2 \le \frac{1}{4}(M-m)^2.
\end{equation}
\end{theorem}

\begin{proof}
    First, notice that

    \begin{align*}
        0 &\le \frac{1}{n}\sum_{i=1}^n (M - p_i)(p_i - m) \\
        &= \frac{1}{n} \sum_{i=1}^n (Mp_i - mM - p_i^2 + mp_i) \\
        &= -mM + \frac{M+m}{n} - \frac{1}{n}\sum_{i=1}^n p_i^2.
    \end{align*}

    Therefore,

    \begin{align*}
        \frac{1}{n}\sum_{i=1}^n (p_i - \frac{1}{n})^2 &= \frac{1}{n}\sum_{i=1}^n p_i^2 - \frac{1}{n} \\
        &\le -mM + \frac{M + m}{n} - \frac{1}{n^2} \\
        &= (M - \frac{1}{n})(\frac{1}{n} - m) \\
        &\le \left(\frac{M - \frac{1}{n} + \frac{1}{n} - m}{2}\right)^2 \\
        &= \frac{1}{4}(M - m)^2,
    \end{align*}

    where the last inequality makes use of the AM-GM inequality. 
    
\end{proof}

\end{document}